\documentclass[12pt]{article}
\usepackage{amsmath}
\usepackage{graphicx,psfrag,epsf}
\usepackage{enumerate}
\usepackage{natbib}
\usepackage{url} 

\usepackage{algorithm}
\usepackage{algorithmic}
\usepackage{amssymb}
\usepackage{bm}

\newtheorem{Thm}{Theorem}
\newtheorem{lemma}{Lemma}
\newtheorem{proof}{Proof}
\newtheorem{definition}{Definition}

\newcommand{\blind}{1}

\begin{document}

\def\spacingset#1{\renewcommand{\baselinestretch}%
{#1}\small\normalsize} \spacingset{1}


\if1\blind
{
  \title{\bf Strictly Proper Kernel Scoring Rules and Divergences   with an Application to Kernel Two-Sample Hypothesis Testing}
  \author{Hamed Masnadi-Shirazi \hspace{.2cm}\\
    School of Electrical and Computer Engineering \\Shiraz University \\Shiraz, Iran}
  \maketitle
} \fi

\if0\blind
{
  \bigskip
  \bigskip
  \bigskip
  \begin{center}
    {\LARGE\bf Strictly Proper Kernel Scoring Rules and Divergences   with an Application to Kernel Two-Sample Hypothesis Testing}
\end{center}
  \medskip
} \fi

\bigskip
\begin{abstract}
We study strictly proper scoring rules in the Reproducing Kernel Hilbert Space. We propose a general Kernel Scoring rule and associated Kernel Divergence. 
We consider  conditions under which the Kernel Score  is strictly proper. We then demonstrate that the Kernel Score includes the Maximum Mean Discrepancy
as a special case. We also consider the connections between the Kernel Score and the minimum risk of a proper loss function. We show that the
Kernel Score incorporates more information pertaining to the projected embedded distributions compared to the Maximum Mean Discrepancy. Finally,
we show how to integrate the information provided from different Kernel Divergences, such as the proposed Bhattacharyya Kernel Divergence, using 
a one-class classifier for improved two-sample hypothesis testing results. 
\end{abstract}

\noindent%
{\it Keywords:}  strictly proper scoring rule,  divergences, kernel scoring rule, minimum risk, projected risk, proper loss functions,  
  probability elicitation, calibration, Bayes error bound, Bhattacharyya distance, feature selection,  maximum mean discrepancy,
	kernel two-sample hypothesis testing, embedded distribution  
\vfill

\newpage
\spacingset{1.45} 
\label{intro}
\section{Introduction}
Strictly proper scoring rules  \cite{savage,DeGroot2, Raftery} are integral to a number of different applications namely, forecasting \cite{Tilmann2007, Brocker2009}, probability elicitation \cite{book:Eliciting},
 classification \cite{HamedNunoLossDesign,HamedNunoJMLRRegularize}, estimation \cite{EstWithScrictLoss}, and finance \cite{Duffie}. Strictly proper scoring rules are closely related to entropy functions, divergence measures
and bounds on the Bayes error that are important for applications such as  feature selection \cite{NUNOMaxDiversityNIPS, NunoNaturalFeatures,NewFeatPers,FeatMutual}, classification and regression \cite{KLboost,ProjectClass,FriedmanPersuit} and information theory \cite{Fano, fdivrisk, book:InfoTheory, minimaxrisk}.

Despite their vast applicability and having been extensively studied, strictly proper scoring rules have only recently been studied in  Reproducing Kernel Hilbert Spaces. In \cite{Dawid2007KScore, Raftery} a certain kernel score is defined and in \cite{Zawadzki} its divergence is shown to be equivalent to the Maximum Mean Discrepancy. The Maximum Mean Discrepancy (MMD) \cite{MMD} is defined as the squared difference
between the embedded means of two  distributions embedded in an inner product kernel space. 
It has been  used in hypothesis testing where the null hypothesis is
rejected if the MMD of two sample sets is above a certain threshold \cite{TwoMMD,MMD}. Recent work pertaining to the  MMD has concentrated on the kernel  
function \cite{Barat3,Barat1,Barat2,  Barat4}  or improved estimates of the mean embedding \cite{Barat5} 
or methods of improving its implementation \cite{Barat6} or  incorporating the embedded covariance \cite{FisherMMD} among others.

In this paper we study the notion of strictly proper scoring rules in the Reproducing Kernel Hilbert Space. We introduce a general Kernel Scoring rule and associated Kernel Divergence that encompasses the MMD and the kernel score of \cite{Dawid2007KScore, Raftery, Zawadzki} as special cases. We then provide conditions under which the proposed Kernel Score is proven to be strictly proper. We show that being strictly proper  is closely related to the injective property of the MMD. 

The Kernel Score is shown to be dependent on the choice of an embedded projection vector $\Phi({\bf w})$ and concave function $C$. 
We consider a number of valid choices of  $\Phi({\bf w})$ such as the canonical vector, the normalized kernel Fisher discriminant projection vector and the 
normalized kernel SVM projection vector \cite{book:vapnik} that lead to strictly proper Kernel Scores and strictly proper Kernel Divergences.

We show that the proposed Kernel Score is related to the minimum risk and that the $C$ is related to the minimum conditional risk function. This connection is
made possible by looking at risk minimization in terms of proper loss functions \cite{Buja, HamedNunoLossDesign, HamedNunoJMLRRegularize}. This allows us to study the effect of choosing different $C$ functions and establish its relation to the Bayes error. We then provide a method for generating $C$ functions for Kernel Scores that are arbitrarily tight upper bounds on the Bayes error. This is especially important for applications that rely on tight bounds on the Bayes error such as classification, feature selection and feature extraction among others. In the experiment section we confirm that such tight bounds on the Bayes error lead to improved feature selection and classification results.

We show that strictly proper Kernel Scores and Kernel Divergences, such as the  Bhattacharyya Kernel Divergence, include more information about the projected embedded distributions compared to the MMD. We provide practical formulations for calculating the Kernel Score and Kernel Divergence and show how to combine the information provided from different Kernel Divergences with the MMD using a one class classifier \cite{OneClassPhd} for significantly improved hypothesis testing results. 

The paper is organized as follows. In Section 2 we review the required background material. In Section 3 we introduce the  Kernel Scoring Rule and Kernel Divergence
and consider conditions under which they are strictly proper. In Section 4 we establish the connections between the Kernel Score and and MMD and show that the MMD is a special case of the Bhattacharyya Kernel Score. In Section 5 we show the connections between the Kernel Score and the minimum risk and explain how arbitrarily tighter bounds on the Bayes error are possible. In Section 6 we discuss  practical consideration in computing the Kernel Score and Kernel Divergence given sample data. 
In Section 7 we propose a novel one-class classifier that can combine all the different Kernel Divergences  into a powerful hypothesis test. Finally, in Section 8 we present extensive experimental results and apply the  proposed ideas to feature selection and hypothesis testing on bench-mark gene data sets.

\section{Background Material Review}
In this section we provide a review of required background material on strictly  proper scoring rules, proper loss functions  and 
positive definite kernel embedding of probability distributions.

\subsection{Strictly  Proper Scoring Rules and Divergences}
The concept of strictly  proper scoring rules can be traced back to the seminal paper of \cite{savage}. This idea was expanded upon by later papers such as
\cite{DeGroot2,Dawid1981} and has been most recently studied under a broader context  \cite{book:Eliciting, Raftery}.
We provide a short review of the main ideas in this field.

Let $\Omega$ be a general sample space and $\cal P$
be a class of probability measures on $\Omega$. A scoring rule $S:{\cal P} \times  \Omega \rightarrow  \mathbb{R}$
is a real valued function that assigns the score  $S(P,x)$ to a forecaster that quotes the measure  $P \in \cal P$ and the
event $a \in \Omega$ materializes. The expected score is written as $S(P,Q)$ and is the expectation of $S(P,.)$ under $Q$ 
\begin{equation}
S(P,Q)=\int S(P,a) dQ(a),
\end{equation}
assuming that the integral exists. We say that a scoring rule is proper if 
\begin{equation}
S(Q,Q) \ge S(P,Q) ~\mbox{for all}~ P,Q
\end{equation}
and we say that a scoring rule is strictly proper when $S(Q,Q)=S(P,Q)$ if and only if $P=Q$.
We define the divergence associated with a strictly proper scoring rule $S$ as
\begin{equation}
div(P,Q)=S(Q,Q)-S(P,Q) \ge 0
\end{equation}
which is a non-negative function and has the property of 
\begin{equation}
\label{eq:StrictDivProperty}
div(P,Q)=0 ~\mbox{if and only if}~ P=Q.
\end{equation}

Presented less formally, the forecaster makes a prediction regarding an event in the form of a probability distribution  $P$. If the actual event 
$a$ materializes then the forecaster is assigned a score of $S(P,a)$. If the true  distribution of events is $Q$ then the expected score is
$S(P,Q)$.   Obviously, we want to assign the maximum score to a skilled and trustworthy forecaster that predicts $P=Q$. A strictly proper score 
accomplishes this by assigning the maximum score if and only if $P=Q$. 

If the distribution of the forecasters predictions is $\nu(P)$ then the overall expected score of the forecaster is 
\begin{equation}
\int \nu(P) S(P,Q) dP.
\end{equation}
The overall expected score is maximum  when the expected score $S(P,Q)$ is maximum for each prediction $P$, which happens when $P=Q$ for all $P$, assuming that the score is strictly proper.

\subsection{Risk Minimization  and the Classification Problem  }
\label{sec:DeriveTangentLoss}
Classifier design by risk minimization has been extensively studied in ~\citep{friedman,zhang,Buja,HamedNunoLossDesign}. 
In summary, a classifier $h$ is defined as a mapping from a feature vector ${\bf x} \in \cal X$ to a class 
label $y \in \{-1,1\}$. Class labels $y$ and feature vectors ${\bf x}$ are sampled from the probability distributions $P_{Y|X}(y|{\bf x})$ and $P_{\bf X}({\bf x})$ respectively. Classification is accomplished by taking the sign of the classifier predictor $p: {\cal X} \rightarrow \mathbb{R}$. This can be written as 
\begin{equation}
  h({\bf x}) = sign[p({\bf x})].
  \label{eq:h}
\end{equation}

The optimal predictor $p^*({\bf x})$ is found by minimizing  
the risk over a non-negative loss 
function $L({\bf x},y)$ and written as
\begin{equation}
  R(p) = E_{{\bf X},Y}[L(p({\bf x}),y)].
  \label{eq:risk}
\end{equation}
This is equivalent to minimizing the conditional risk 
\begin{equation*}
  E_{Y|{\bf X}} [L(p({\bf x}),y)|{\bf X} = {\bf x}]
\end{equation*}
for all ${\bf x} \in {\cal X}$. 
The predictor $p({\bf x})$ is decomposed and typically written as
\begin{equation*}
  p({\bf x}) = f(\eta({\bf x})),
  \label{eq:compose}
\end{equation*}
where $f: [0,1] \rightarrow \mathbb{R}$ is called the link function and 
 $\eta({\bf x}) = P_{Y|{\bf X}}(1|{\bf x})$ is the posterior 
probability function.
The optimal predictor can now be learned by first analytically finding the optimal link $f^*(\eta)$
and then estimating $\eta({\bf x})$, assuming that $f^*(\eta)$ is  one-to-one.

If the zero-one loss 
\begin{eqnarray*}
  L_{0/1}(y,p) = \frac{1- sign(yp)}{2}  = \left\{ \begin{array}{ll}
         0, & \mbox{if $y=sign(p)$};\\
        1, & \mbox{if $y \ne sign(p)$},\end{array} \right.
\end{eqnarray*}
is used, then the associated  conditional risk 
\begin{eqnarray}
\label{eq:zeronecondrisk}
  C_{0/1}(\eta,p) = \eta \frac{1- sign(p)}{2} + 
  (1-\eta) \frac{1 + sign(p)}{2}
    = \left\{ \begin{array}{ll}
         1-\eta, & \mbox{if $p=f(\eta) \geq 0 $};\\
        \eta, & \mbox{if $p=f(\eta)<0$} \end{array} \right.
\end{eqnarray}

is equal to the probability of error of the classifier of~(\ref{eq:h}). The associated  conditional zero-one risk  is 
minimized by any $f^*$ such that
\begin{equation}
  \left\{
  \begin{array}{cc}
    f^*(\eta) > 0 & \mbox{if $\eta > \frac{1}{2} $} \\
    f^*(\eta) = 0 & \mbox{if $\eta =  \frac{1}{2} $} \\
    f^*(\eta) < 0 & \mbox{if $\eta <  \frac{1}{2} $.}
  \end{array}
  \right.
  \label{eq:Bayesnec}
\end{equation}
For example the two links of
\begin{equation*}
  f^*=2\eta-1 \quad \mbox{and} \quad f^*=\log\frac{\eta}{1-\eta}
  \label{eq:fexamples}
\end{equation*}
can be used.

The resulting classifier $h^*({\bf x}) = sign[f^*(\eta({\bf x}))]$
is now the optimal Bayes decision 
rule.  Plugging $f^*$ back into the conditional zero-one risk gives the minimum conditional zero-one risk 
\begin{eqnarray}
\label{eq:zeronemincondrisk}
  C^*_{0/1} (\eta) &&= \eta\left(\frac{1}{2}-\frac{1}{2}sign(2\eta-1)\right)+ 
           (1-\eta)\left(\frac{1}{2}+\frac{1}{2}sign(2\eta-1)\right) \\
					&&=\left\{ \begin{array}{ll}
         (1-\eta) & \mbox{if $\eta \geq \frac{1}{2} $}\\
        \eta & \mbox{if $\eta<\frac{1}{2}$}\end{array} \right. \\
           &&=\min\{\eta, 1-\eta\}. 
\end{eqnarray}

The optimal classifier that is found using the zero-one loss has the 
smallest possible risk and is known as the Bayes error  $R^*$ of the corresponding classification problem ~\citep{JordanBartlett, zhang, book:ProbPatRec}.

We can change the loss function and replace the zero-one loss with a so-called margin loss in the form of 
$L_{\phi}(y,p({\bf x})) = \phi(yp({\bf x}))$. 
Unlike the zero-one loss, margin losses allow for a non-zero loss on  positive values of
the margin $yp$. Such loss functions can be shown to produce classifiers that have better generalization ~\citep{book:vapnik}.
Also unlike the zero-one loss, margin losses are typically designed to be differentiable over their entire domain. 
The exponential loss and logistic loss used in the AdaBoost and LogitBoost Algorithms \cite{friedman} and the hinge loss used in SVMs  are some examples of margin losses \cite{zhang,Buja}.
The conditional risk of a margin loss can now be written as
\begin{equation}
  C_\phi(\eta,p) = C_\phi(\eta,f(\eta)) = \eta \phi(f(\eta)) + 
  (1-\eta) \phi(-f(\eta)).
  \label{eq:CondRisk}
\end{equation}
This is minimized by the link
\begin{equation}
  f^*_{\phi}(\eta) = \arg\min_{f} C_\phi(\eta,f)
  \label{eq:fstarphi}
\end{equation}
and so the minimum conditional risk 
function is
\begin{equation}
  C^*_\phi(\eta) = C_\phi(\eta,f^*_\phi).
  \label{eq:C*phi}
\end{equation}
For most margin losses, the optimal link is unique and can be found analytically. Table~\ref{tab:losses} presents the exponential, logistic and hinge losses along with their respective link  and minimum conditional risk functions.

\begin{table}[t]
  \centering
  \caption{\protect\footnotesize{Loss $\phi$, optimal link $f^*_{\phi}(\eta)$,
      optimal inverse link $[f^*_{\phi}]^{-1}(v)$ , 
      and minimum conditional risk $C_\phi^*(\eta)$ of popular learning
      algorithms.}}
  \begin{tabular}{|c|c|c|c|c|}
    \hline
    Algorithm & $\phi(v)$ & $f^*_{\phi}(\eta)$ & $[f^*_{\phi}]^{-1}(v)$ & 
    $C_{\phi}^*(\eta)$ \\
		\hline
    \hline
    AdaBoost & $\exp(-v)$ & $\frac{1}{2} \log \frac{\eta}{1-\eta}$& 
    $\frac{e^{2v}}{1+e^{2v}}$ & $2 \sqrt{\eta (1-\eta)}$\\
    LogitBoost & $\log(1+e^{-v})$ & $\log \frac{\eta}{1-\eta}$ & 
    $\frac{e^{v}}{1+e^{v}}$ &
    $-\eta \log \eta - (1-\eta) \log (1 - \eta)$\\
		SVM & $\max(1-v,0)$ & $sign(2 \eta - 1)$ & NA & $1 - |2\eta -1|$\\
    \hline
  \end{tabular}
  \label{tab:losses}
\end{table}

\subsubsection{Probability Elicitation and Proper Losses}
Conditional risk minimization can be related to  probability elicitation  ~\citep{savage,DeGroot} and has  been studied in ~\citep{Buja,HamedNunoLossDesign,Reid}. 
 In  probability elicitation we 
find the probability estimator ${\hat \eta}$ that maximizes the expected score
\begin{equation}
  I(\eta,{\hat \eta}) = \eta I_{1}({\hat \eta}) + (1-\eta) I_{-1}({\hat \eta}),
  \label{eq:expreward}
\end{equation}
of a score function that assigns a score of $I_1({\hat \eta})$   to prediction ${\hat \eta}$ 
when event $y=1$ holds and a score of
$I_{-1}({\hat \eta})$ to prediction ${\hat \eta}$ when $y=-1$ holds. The scoring function is said to be proper if 
$I_1$ and $I_{-1}$ are such that the 
expected score is maximal when ${\hat \eta} = \eta$, in other words
\begin{equation}
  I(\eta,{\hat \eta}) \leq I(\eta, \eta) = J(\eta), \,\,\, \forall \eta
  \label{eq:Savagebound}
\end{equation}
with equality if and only if ${\hat \eta} = \eta$. This holds for the following theorem.
\begin{Thm}{~\citep{savage}}
  Let $I(\eta,{\hat \eta})$ be as defined 
  in~(\ref{eq:expreward}) and $J(\eta) = I(\eta,\eta)$.
  Then~(\ref{eq:Savagebound}) holds if and only if $J(\eta)$ is convex and

  \begin{equation}
    \label{eq:Is}
    I_1(\eta) = J(\eta) + (1-\eta) J^\prime(\eta) \quad \quad \quad
    I_{-1}(\eta) = J(\eta) -\eta J^\prime(\eta).
  \end{equation}
  \label{thm:savage}
\end{Thm}

Proper losses can now be related to  probability elicitation by the following theorem which is most important for our purposes.
\begin{Thm}{~\citep{HamedNunoLossDesign}}
\label{Thm:HamedNuno} Let $I_1(\cdot)$ and 
  $I_{-1}(\cdot)$ be as in (\ref{eq:Is}),
  for any continuously differentiable convex $J(\eta)$ such that
  $J(\eta) = J(1-\eta)$, and $f(\eta)$ any invertible function such that
  $f^{-1}(-v) = 1 -  f^{-1}(v)$.
  Then
  \begin{equation*}
    I_1(\eta) = -\phi(f(\eta)) \quad \quad \quad \quad \quad \quad 
    I_{-1}(\eta) = -\phi(-f(\eta)) \label{eq:I1I-1f}
  \end{equation*}
  if and only if 
  \begin{equation*}
    \phi(v) =  -J\left(f^{-1}(v)\right) - (1- f^{-1}(v)) J^\prime 
    \left(f^{-1}(v)\right).
    \label{eq:phieq}
  \end{equation*}
  \label{thm:risk}
\end{Thm}
It is shown in \citep{zhang}  that $C_\phi^*(\eta)$ is concave  and that
\begin{eqnarray}
  C_\phi^*(\eta) &=&   C_\phi^*(1-\eta) \label{eq:Cstarsym}\\
  {[f_\phi^*]}^{-1}(-v) &=& 1 - [f_\phi^*]^{-1}(v) \label{eq:fstarsym}.
\end{eqnarray}
We also require  that $C^*_\phi(0) = C^*_\phi(1) = 0$ so that the 
minimum risk is zero when  $P_{Y|{\bf X}}(1|{\bf x}) = 0$ or
$P_{Y|{\bf X}}(1|{\bf x}) = 1$.

In summary, for any continuously differentiable $J(\eta) = -C_\phi^*(\eta)$ and invertible $f(\eta) = 
f^*_\phi(\eta)$,  the conditions of  Theorem \ref{Thm:HamedNuno} are satisfied and so the loss will take the form of 
\begin{equation}
  \phi(v) =  C_\phi^*\left([f_\phi^*]^{-1}(v)\right) + (1- [f_\phi^*]^{-1}(v)) 
  [C_\phi^*]^\prime\left([f_\phi^*]^{-1}(v)\right)
  \label{eq:phieq2}
\end{equation}
 and $I(\eta,{\hat \eta}) = -C_\phi(\eta,f)$.
In this case, the predictor of minimum risk is $p^* = f^*_\phi(\eta)$, the  minimum risk is 
\begin{equation}
\label{equ:MinRiskInit}
R({p^*}) = \int_{\bf x} P_{{\bf X}}({\bf x}) \left[ P_{{\bf Y}|{\bf X}}(1|{\bf x})\phi({ p^*}({\bf x})) + P_{{\bf Y}|{\bf X}}(-1|{\bf x})\phi(-{ p^*}({\bf x}))  \right] d{\bf x}
\end{equation}
and posterior probabilities $\eta$ can be found using 
\begin{equation}
  \eta({\bf x}) = [f^*_\phi]^{-1}( p^*({\bf x})).
  \label{eq:link}
\end{equation} 
Finally, the loss is said to be proper
and the predictor calibrated~\citep{DeGroot, Platt, Caruana, Raftery}. 

In practice, an estimate  
of the optimal predictor ${\hat p}^*({\bf x})$ is found by minimizing the empirical risk 
\begin{equation}
  R_{emp}(p) = \frac{1}{n} \sum_i L(p({\bf x}_i),y_i)
  \label{eq:emprisk}
\end{equation}
over  a training set ${\cal D} = \{({\bf x}_1,y_1), \ldots, ({\bf x}_n,y_n)\}$.
Estimates of the probabilities $ \eta({\bf x})$ are now found from  ${\hat p}^*$ using
\begin{equation}
  {\hat \eta}({\bf x}) = [f^*_\phi]^{-1}({\hat p^*}({\bf x})).
  \label{eq:hateta}
\end{equation}

\subsection{Positive Definite Kernel Embedding of Probability Distributions }
In this section we review the notion of embedding probability measures into reproducing kernel Hilbert spaces \cite{book:RHKS, Fukumizu2004, Sriperumbudur2010}.

Let ${\bf x} \in \cal X$ be a random variable defined on a topological space $\cal X$ with associated probability measure $P$. 
Also, let $\cal H$  be a  Reproducing Kernel Hilbert Space (RKHS) .  Then there is a mapping $\Phi: {\cal X} \rightarrow {\cal H}$
such that 
\begin{equation}
<\Phi({\bf x}),f>_{\cal H}=f({\bf x}) ~\mbox{for all}~ f \in {\cal H}.
\end{equation}
The mapping can be written as $\Phi({\bf x}) = k({\bf x},.)$ where $k(.,{\bf x})$ is a positive definite kernel function parametrized by ${\bf x}$. A  dot product representation of $k({\bf x},{\bf y})$ exists in the form of 
\begin{equation}
k({\bf x},{\bf y})=<\Phi({\bf x}),\Phi({\bf y})>_{\cal H}
\end{equation}
where ${\bf x},{\bf y} \in {\cal X}$.   

For a given  Reproducing Kernel Hilbert Space $\cal H$, the mean embedding ${\bm \mu}_P \in {\cal H}$ of the  distribution $P$ exists under certain conditions and  is defined as
\begin{equation}
{\bm \mu}_P(t)=<{\bm \mu}_P(.),k(.,t)>_{\cal H}=E_{\cal X}[k(x,t)].
\end{equation}
In words, the mean embedding ${\bm \mu}_P$ of the distribution $P$ is the expectation under $P$ of the  mapping $k(.,t)=\Phi(t)$. 

The   maximum mean discrepancy (MMD) \cite{MMD} is expressed as the squared difference between
the embedded means ${\bm \mu}_P$ and ${\bm \mu}_Q$ of the two embedded distributions $P$ and $Q$ as
\begin{eqnarray}
\label{eq:MMDDEf}
MMD_{\cal F}(P,Q)=||{\bm \mu}_P - {\bm \mu}_Q||^2_{\cal H}.
\end{eqnarray}
where $\cal F$ is a unit ball in a universal RKHS which requires that $k(.,x)$ be continuous among other things. It can be shown that the Reproducing Kernel Hilbert Spaces associated with the Gaussian and Laplace kernels are universal \cite{Steinwart}. 
Finally, an important property of the MMD is that it is injective which is formally stated by the following theorem.
\begin{Thm}{~\citep{MMD}}
\label{Thm:InjectiveMMD}
Let $\cal F$ be a unit ball in a universal RKHS $\cal H $ defined on the compact metric space $\cal X$ with associated continuous kernel $k(.,x)$. 
$MMD_{\cal F}(P,Q) = 0$ if and only if $P=Q$.
\end{Thm}

\section{Strictly Proper Kernel Scoring Rules and Divergences }

In this section we define the Kernel Score and Kernel Divergence and show when the Kernel Score is strictly proper. 
To do this we need to define the projected embedded distribution.
\begin{definition}
Let ${\bf x} \in \cal X$ be a random variable defined on a topological space $\cal X$ with associated probability distribution $P$. 
Also, let $\cal H$  be a  universal RKHS with associated positive definite kernel function 
$k({\bf x},{\bf w})=<\Phi({\bf x}),\Phi({\bf w})>_{\cal H}$. The projection of $\Phi({\bf x})$ onto a fixed vector $\Phi({\bf w})$ in $\cal H$ is denoted by $x^p$
and found as
\begin{eqnarray}
\label{eq:FindxpwK}
x^p=\frac{k({\bf w},{\bf x})}{\sqrt{k({\bf w},{\bf w})}}.
\end{eqnarray}
The univariate distribution associated with $x^p$ is defined as the projected embedded distribution of $P$ and denoted by $P^p$. The mean and variance of $P^p$
are denoted by $\mu^p_P$ and $(\sigma^p_P)^2$.
\end{definition}

The Kernel Score and Kernel Divergence  are now defined as follows.
\begin{definition}
Let $P$ and $Q$ be two distributions on $\cal X$. Also, let $\cal H$  be a  universal RKHS with associated positive definite kernel function 
$k({\bf x},{\bf w})=<\Phi({\bf x}),\Phi({\bf w})>_{\cal H}$ where $\cal F$ is a unit ball in $\cal H$. 
Finally, assume that $\Phi({\bf w})$  is a fixed vector  in $\cal H$.
The Kernel Score  between distributions $P$ and $Q$ is defined as 
\begin{eqnarray}
S_{C,k,{\cal F},{\Phi({\bf w})}}(P,Q)= \int  \left( \frac{P^p(x^p) + Q^p(x^p)}{2} \right) C\left( \frac{P^p(x^p)}{P^p(x^p) + Q^p(x^p)}  \right)d(x^p),
\end{eqnarray}
and the Kernel Divergence between distributions $P$ and $Q$ is defined as 
\begin{eqnarray}
KD_{C,k,{\cal F},{\Phi({\bf w})}}(P,Q) &=& \frac{1}{2} -  S_{C,k,{\cal F},{\Phi({\bf w})}}(P,Q) \\
&=&\frac{1}{2} - \int  \left( \frac{P^p(x^p) + Q^p(x^p)}{2} \right) C\left( \frac{P^p(x^p)}{P^p(x^p) + Q^p(x^p)}  \right)d(x^p),
\end{eqnarray}
where $C$ is a continuously differentiable strictly concave symmetric function such that $C(\eta)=C(1-\eta)$ for all $\eta \in [0~1]$, $C(0)=C(1)=0$, $C(\frac{1}{2})=\frac{1}{2}$ and $P^p$ and $Q^p$ are the projected embedded distributions of $P$ and $Q$.
\end{definition}

We can now present conditions under which a Kernel Score is strictly proper and Kernel Divergence 
has the important property of (\ref{eq:StrictDivProperty}). 
\begin{Thm}
\label{thm:StrictDivProperty}
The Kernel Score  is strictly proper and the Kernel Divergence   has the property of
\begin{equation}
\label{eq:StrictKernelDivProperty}
KD_{C,k,{\cal F},{\Phi({\bf w})}}(P,Q)=0 ~\mbox{if and only if}~ P=Q,
\end{equation}
if $\Phi({\bf w})$ is chosen such that it is not in the orthogonal compliment of the set $M=\{{\bm \mu}_P - {\bm \mu}_Q\}$, where ${\bm \mu}_P$ and ${\bm \mu}_Q$ are the mean embeddings of $P$ and $Q$ respectively.
\end{Thm}
\begin{proof}
See supplementary material \ref{app:KscoreStrict}.
\end{proof}
We denote Kernel Divergences that have the desired property of (\ref{eq:StrictKernelDivProperty}) as Strictly Proper Kernel Divergences. 
The  canonical  projection vector $\Phi({\bf w})$ that is not in the orthogonal compliment of $M=\{{\bm \mu}_P - {\bm \mu}_Q\}$
is to choose $\Phi({\bf w})=\frac{({\bm \mu}_P-{\bm \mu}_Q)}{||({\bm \mu}_P-{\bm \mu}_Q)||_{\cal H}}$. The following lemma lists some  valid choices.
\begin{lemma}
\label{thm:validchoices}
The Kernel Score and Kernel Divergence  associated with the following choices of $\Phi({\bf w})$ are strictly proper.
\begin{enumerate}
\item
$\Phi({\bf w})=\frac{({\bm \mu}_P-{\bm \mu}_Q)}{||({\bm \mu}_P-{\bm \mu}_Q)||_{\cal H}}$.
\item
$\Phi({\bf w})$ equal to the normalized kernel Fisher discriminant projection vector.
\item
$\Phi({\bf w})$ equal to the normalized kernel SVM  projection vector.
\end{enumerate}
\end{lemma} 
\begin{proof}
See supplementary material \ref{app:ValidVect}.

%
\end{proof}
In what follows we consider the implications of choosing different $\Phi({\bf w})$ projections and concave functions $C$ for the Strictly Proper Kernel Score and Kernel Divergence.

\section{The Maximum Mean Discrepancy Connection}
If we choose $C$ to be the concave function of $C_{Exp}(\eta)=\sqrt{(\eta(1-\eta))}$ and assume that the univariate projected embedded distributions $P^p$ and $Q^p$ are Gaussian then, using the 
Bhattacharyya bound \cite{BattLeeFeat,ImageSegClust}, we can readily show that 
\begin{eqnarray}
\label{eq:battarExp2}
&&S_{C,k,{\cal F},{\Phi({\bf w})}}(P,Q)= \frac{1}{2}\cdot e^{(B)}, \\
&&KD_{C_{Exp},k,{\cal F},{\Phi({\bf w})}}(P,Q)=\frac{1}{2} -\frac{1}{2}\cdot e^{(B)}, \\
&& B=\frac{1}{4}\log \left ( \frac{1}{4} \left (\frac{( \sigma^p_P)^2}{( \sigma^p_Q)^2}+\frac{( \sigma^p_Q)^2}{( \sigma^p_P)^2}+2 \right)\right )+\frac{1}{4} \left ( \frac{( \mu^p_P- \mu^p_Q)^2}{( \sigma^p_P)^2 + ( \sigma^p_Q)^2} \right),
\end{eqnarray}
where $ \mu^p_P$, $ \mu^p_Q$, $\sigma^p_P$ and $ \sigma^p_Q$  are the means and variances of the projected embedded distributions $P^p$ and $Q^p$.
We will refer to these as the Bhattacharyya Kernel Score and Bhattacharyya Kernel Divergence. 
Note that if $\sigma^p_P = \sigma^p_Q$ then the above equation simplifies to 
$B=\frac{1}{4} \left ( \frac{( \mu^p_P- \mu^p_Q)^2}{( \sigma^p_P)^2 + ( \sigma^p_Q)^2} \right)$. 
This leads to the following results.

\begin{lemma}
\label{lemma:AlternateMMD}
Let $P$ and $Q$ be two  distributions where  $\mu^p_P$ and  $\mu^p_Q$ are the  respective  means  of the projected embedded distributions $P^p$ and $Q^p$ with projection vector $\Phi({\bf w})=\frac{({\bm \mu}_P-{\bm \mu}_Q)}{||({\bm \mu}_P-{\bm \mu}_Q)||_{\cal H}}$. Then 
\begin{eqnarray}
MMD_{\cal F}(P,Q)=(\mu^p_P - \mu^p_Q)^2.
\end{eqnarray}
\end{lemma} 
\begin{proof}
See supplementary material \ref{app:MMDNewLook}.
\end{proof}
With this new alternative outlook on the MMD, it can be seen as a special case of a strictly proper Kernel Score under certain assumptions outlined in the following theorem.

\begin{Thm}
\label{Thm:KScoreMMDConnection}
Let $C$  be the concave function of $C_{Exp}(\eta)=\sqrt{(\eta(1-\eta))}$ and $\Phi({\bf w})=\frac{({\bm \mu}_P-{\bm \mu}_Q)}{||({\bm \mu}_P-{\bm \mu}_Q)||_{\cal H}}$. Then 
\begin{eqnarray}
MMD_{\cal F}(P,Q) \propto \log \left(2 S_{C_{Exp},k,{\cal F},{\Phi({\bf w})}}(P,Q) \right)
\end{eqnarray}
under the assumption that the projected embedded distributions $P^p$ and $Q^p$ are Gaussian distributions of equal variance.
\end{Thm}
\begin{proof}
See supplementary material \ref{app:MMDKSCoreConnection}.
\end{proof}

In other words, if we set $\Phi({\bf w})=\frac{({\bm \mu}_P-{\bm \mu}_Q)}{||({\bm \mu}_P-{\bm \mu}_Q)||}$ and project  onto this
vector, the MMD is equal to the distance between the means of the projected embedded distributions  squared. 
Note that while the MMD incorporates all the higher moments of the distribution of the data in the original space and determines a probability distribution 
uniquely  \cite{KernelEmbeddingBeyonds}, it 
completely disregards the higher moments of the  projected embedded distributions.
This suggests that by incorporating more information regarding the projected embedded distributions, such as its variance, we can arrive at measures such as the Bhattacharyya Kernel Divergence that are more versatile than the MMD in the finite sample setting.  In the experimental section we apply these  measures to the problem of kernel hypothesis testing and show that they outperform the MMD.

\section{Connections to the Minimum  Risk}
In this section we establish the connection between the Kernel Score and the  minimum risk associated with the projected embedded distributions. 
This will provide further insight towards the effect of choosing different concave $C$ functions and different projection vectors $\Phi({\bf w})$ on the Kernel Score.
First, we present a  general formulation for the minimum risk of (\ref{eq:risk}) for a proper loss function and show that  we can partition any such risk into two terms
akin to partitioning of the Brier score~\citep{DeGroot, Murphy1972}.
\begin{lemma}
\label{Thm:Brier}
Let $\phi$ be a proper loss  function in the form of (\ref{eq:phieq2}) and ${\hat p^*}({\bf x})$ an estimate of the optimal predictor ${ p^*}({\bf x})$. The  risk $R({\hat p^*})$ can be partitioned into a term that is a measure of calibration $R_{Calibration}$ plus a term that is the minimum risk $R({p^*})$ in the form of 
\begin{eqnarray}
R({\hat p^*}) &=&  \\
	&& \int_{\bf x} P_{{\bf X}}({\bf x}) \left[ P_{{\bf Y}|{\bf X}}(1|{\bf x})\left(\phi({\hat p^*}({\bf x}))-\phi({ p^*}({\bf x}))\right) + P_{{\bf Y}|{\bf X}}(-1|{\bf x})\left(\phi(-{\hat p^*}({\bf x}))-\phi(-{ p^*}({\bf x})) \right) \right] d{\bf x} \nonumber \\
	&+& \int_{\bf x} P_{{\bf X}}({\bf x}) \left[ P_{{\bf Y}|{\bf X}}(1|{\bf x})\phi({ p^*}({\bf x})) + P_{{\bf Y}|{\bf X}}(-1|{\bf x})\phi(-{ p^*}({\bf x}))  \right] d{\bf x} \nonumber \\
&=&	R_{Calibration} + R({p^*}).
\end{eqnarray}
 Furthermore the minimum risk term $R({p^*})$ can be written as 
\begin{eqnarray}
\label{eq:MinRiskClean}
R({p^*})=\int_{\bf x} P_{{\bf X}}({\bf x})  C_\phi^*(P_{Y|{\bf X}}(1|{\bf x}) ) d{\bf x}.
\end{eqnarray}

\end{lemma}
\begin{proof}
See supplementary material \ref{app:Brier}.
\end{proof}

The following theorem that writes the Kernel Score in terms of the  minimum risk associated with the projected embedded distributions $R^p({p^*})$  is now readily proven.
\begin{Thm}
\label{Thm:KDMinRiskConect}
Let $P$ and $Q$ be two  distributions and choose $C=C^*_{\phi}$. Then
\begin{eqnarray}
\label{eq:KDMinRiskequality}
S_{C^*_{\phi},k,{\cal F},{\Phi({\bf w})}}(P,Q)=  R^p({p^*}),
\end{eqnarray}
where $R^p({p^*})$ is the minimum risk associated with the projected embedded distributions of $P^p$ and $Q^p$.
\end{Thm}
\begin{proof}
See supplementary material \ref{app:KscoreMinRisk}.
\end{proof}
We conclude that the minimum risk associated with the projected embedded distributions term $R^p({p^*})$, and in turn the Kernel Score $S_{C^*_{\phi},k,{\cal F},{\Phi({\bf w})}}(P,Q)$,
 are constants  related to the distributions $P^p$ and $Q^p$ 
(determined by the choice of $\Phi({\bf w})$) and the choice of $C_\phi^*$.

The effect of changing $C=C^*_{\phi}$ can now be studied in detail by noting the general result presented in the following theorem \cite{ATightUpperBound, ArbitrarilyTight, book:ProbPatRec}. 
\begin{Thm}
\label{thm:tightestC}
Let $C_\phi^*$ be a continuously differentiable concave symmetric function such that $C_\phi^*(\eta)=C_\phi^*(1-\eta)$ for all $\eta \in [0~1]$, $C_\phi^*(0)=C_\phi^*(1)=0$ and $C_\phi^*(\frac{1}{2})=\frac{1}{2}$. Then $C_\phi^*(\eta) \ge \min(\eta,1-\eta)$ and $R({p^*}) \ge R^*$. Furthermore, for any $\epsilon$ such that 
$R({p^*})-R^* \le \epsilon$ there exists $\delta$ and $C_\phi^*$ where $C_\phi^*(\eta) - \min(\eta,1-\eta) \le \delta$. 
\end{Thm}
\begin{proof}
See Section 2 of \cite{ATightUpperBound}, Section 2, Theorems 2 and 4 of \cite{ArbitrarilyTight},  and Chapter 2 of \cite{book:ProbPatRec}.
\end{proof}
The above theorem, when especially applied to the projected embedded distributions, states that the minimum risk associated with the projected embedded 
distributions $R^p({p^*})$ is an upper bound on the Bayes risk associated with the projected embedded distributions ${R^{p}}^*$ 
and as $C_\phi^*$ is made arbitrarily close to $C_{0/1}^*=\min(\eta,1-\eta)$ this upper bound is  tight.

In summary, using different $\Phi({\bf w})$ in the Kernel Score formulation, changes the projected embedded distributions of $P^p$ and $Q^p$ and the Bayes risk associated with these projected embedded distributions ${R^{p}}^*$. Using different $C_\phi^*$ changes the upper bound estimate of this Bayes risk $R^p({p^*})$.

\subsubsection{ Tighter Bounds on the Bayes Error}
\label{sec:TighterBounds22}
We can easily verify that, in general,  the minimum risk is equal to the Bayes error when $C_\phi^*=C_{0/1}^*=\min(\eta,1-\eta)$, leading to the smallest possible minimum risk for fixed data distributions. Unfortunately, $C_{0/1}^*=\min(\eta,1-\eta)$ is not continuously differentiable and so we consider other $C_\phi^*$ functions.
For example  when $C_{LS}^*(\eta)=-2\eta(\eta-1)$ is used,
the minimum risk simplifies to
\begin{eqnarray}
R_{C_{LS}^*}(p^*)=\int \frac{P_{{\bf X}|Y}({\bf x}|1)P_{{\bf X}|Y}({\bf x}|-1)}{(P_{{\bf X}|Y}({\bf x}|1)+P_{{\bf X}|Y}({\bf x}|-1))} d{\bf x},
\end{eqnarray}
which  is equal to the asymptotic nearest neighbor bound \cite{book:Fukunaga,NNClassification} on the Bayes error. 
We have used the notation $R_{C_{LS}^*}(p^*)$ to make it clear that this is the minimum risk associated with the $C_{LS}^*$ function.

\begin{figure*}[t]
  \centering
		\includegraphics[width=4in]{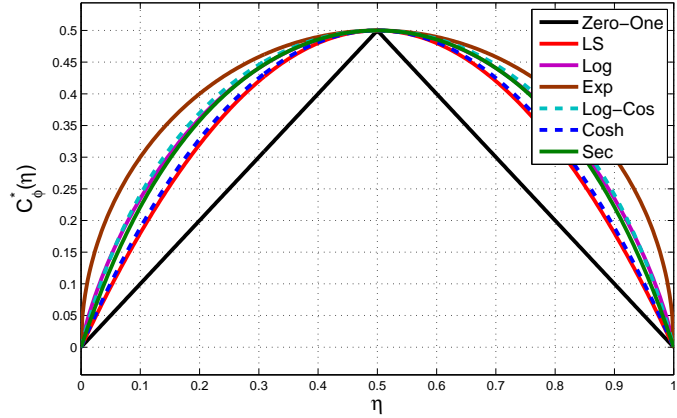}
 \caption{\protect{\footnotesize{Plot of the $C^*_{\phi}(\eta)$ in Table-\ref{tab:JTableParameteres}.}}}
  \label{fig:JTableParamPlot}
\end{figure*}

From Theorem \ref{thm:tightestC} we know that when the minimum risk is computed under  other $C_{\phi}^*$ functions, a list of which is presented in 
Table-\ref{tab:JTableParameteres},  an upper bound on  the Bayes error is  being computed. 
Also, the $C_{\phi}^*$ that are closer to  $C_{0/1}^*$ result in minimum risk formulations that provide tighter bounds on the Bayes error.  Figure-\ref{fig:JTableParamPlot} shows that  $C_{LS}^*$, $C_{Cosh}^*$, $C_{Sec}^*$, $C_{Log}^*$, $C_{Log-Cos}^*$ and $C_{Exp}^*$ are in order the closest to  $C_{0/1}^*$ and the corresponding minimum-risk formulations in Table-\ref{tab:RefFormulasDiffJ4} provide, in the same order, tighter bounds on the Bayes error. This can also be directly verified by noting that $R_{C_{Exp}^*}$ is equal to the Bhattacharyya bound \cite{book:Fukunaga}, $R_{C_{LS}^*}$ is equal to the asymptotic nearest neighbor bound \cite{book:Fukunaga,NNClassification}, $R_{C_{Log}^*}$ is equal to the Jensen-Shannon divergence \cite{JenShannonLin} and $R_{C_{Log-Cos}^*}$ is similar to the bound in \cite{ArbitrarilyTight}. These four formulations have been independently studied in the literature and the fact that they produce upper bounds on the Bayes error has been directly verified. Here we have rederived these four measures by resorting to the concept of minimum risk and proper loss functions which not only allows us to provide a  unified approach to these different methods but has also led to a systematic method for deriving other novel  bounds on the Bayes error, namely  $R_{C_{Cosh}^*}$ and  $R_{C_{Sec}^*}$. 

\begin{table}[tbp]
  \centering
  \caption{\protect\footnotesize{ $C^*_{\phi}(\eta)$ specifics used to compute the minimum-risk.}}
  \begin{tabular}{|c|c|}
    \hline
    Method & $C^*_{\phi}(\eta)$  \\
		\hline
    \hline
    LS & $-2\eta(\eta-1)$ \\
    \hline
    Log & $-0.7213(\eta\log(\eta)-(1-\eta)\log(1-\eta))$ \\
    \hline
    Exp & $\sqrt{\eta(1-\eta)}$ \\
    \hline
    Log-Cos & $(\frac{1}{2.5854})\log(\frac{\cos(2.5854(\eta-\frac{1}{2}))}{\cos(\frac{2.5854}{2})})$ \\
    \hline
    Cosh & $-\cosh(1.9248(\frac{1}{2}-\eta))+\cosh(\frac{-1.9248}{2})$ \\
    \hline
    Sec &  $-\sec(1.6821(\frac{1}{2}-\eta))+\sec(\frac{-1.6821}{2})$ \\
    \hline
  \end{tabular}
  \label{tab:JTableParameteres}
\end{table}


\begin{table}[tbp]
\centering
\caption{\protect\footnotesize{Minimum-risk  for different $C^*_{\phi}(\eta)$ }}
\resizebox{\textwidth}{!}{ 
\begin{tabular}{|c|c|}
    \hline
    $C^*_{\phi}(\eta)$ & $R_{C_{\phi}^*}$ \\
		\hline
    \hline
    Zero-One & Bayes Error \\
    &  \\
    \hline
    LS & $\int \frac{P({\bf x}|1)P({\bf x}|-1)}{P({\bf x}|1)+P({\bf x}|-1)} d{\bf x}$ \\
    &  \\
    \hline
    Exp & $\frac{1}{2}\int \sqrt{P({\bf x}|1)P({\bf x}|-1)} d{\bf x}$ \\
    &  \\
    \hline
    Log & $-\frac{0.7213}{2}D_{KL}(P({\bf x}|1)||P({\bf x}|1)+P({\bf x}|-1))-\frac{0.7213}{2}D_{KL}(P({\bf x}|-1)||P({\bf x}|1)+P({\bf x}|-1))$ \\
    &  \\
    \hline
    Log-Cos & $\int \frac{P({\bf x}|1)+P({\bf x}|-1)}{2} \left[ \frac{1}{2.5854}\log\left(\frac{\cos(\frac{2.5854(P({\bf x}|1)-P({\bf x}|-1))}{2(P({\bf x}|1)+P({\bf x}|-1))})}{cos(\frac{2.5854}{2})}\right) \right] d{\bf x}$ \\
    &  \\
    \hline
    Cosh & $\int \frac{P({\bf x}|1)+P({\bf x}|-1)}{2} \left[ -\cosh(\frac{1.9248(P({\bf x}|-1)-P({\bf x}|1))}{2(P({\bf x}|1)+P({\bf x}|-1))}) +\cosh(\frac{-1.9248}{2}) \right] d{\bf x}$ \\
    &  \\
    \hline
    Sec & $\int \frac{P({\bf x}|1)+P({\bf x}|-1)}{2} \left[ -\sec(\frac{1.6821(P({\bf x}|-1)-P({\bf x}|1))}{2(P({\bf x}|1)+P({\bf x}|-1))}) +\sec(\frac{-1.6821}{2}) \right]  d{\bf x}$ \\
    &  \\
    \hline
\end{tabular}
}
\label{tab:RefFormulasDiffJ4}
\end{table}

Next, we   demonstrate a general procedure for deriving a class of polynomial functions   $C_{Poly-n}^*(\eta)$ that are increasingly  and arbitrarily close 
to $C_{0/1}^*(\eta)$. 
\begin{Thm}
\label{Thm:TightestBoundCR}
Let 
\begin{eqnarray}
C_{Poly-n}^*(\eta)=K_2(\int  Q(\eta) d(\eta) +K_1\eta) 
\end{eqnarray}
where 
\begin{eqnarray}
&&Q(\eta) = \int-(\eta(1-\eta))^n d(\eta), \\
&&K_1 = -Q(\frac{1}{2}), \\
&&K_2 = \frac{\frac{1}{2}}{\left. (\int  Q(\eta) d(\eta) +K_1\eta)\right|_{\eta=\frac{1}{2}}}.
\end{eqnarray}
Then 
$R_{C_{Poly-n}^*} \ge R_{C_{Poly-(n+1)}^*} \ge R^*$ for all $n \ge 0$ and $R_{C_{Poly-n}^*}$ converges to $R^*$ as $n \rightarrow \infty$. 
\end{Thm}
\begin{proof}
See supplementary material \ref{app:Polyn}.
\end{proof}

As an example, we  derive $C_{Poly-2}^*(\eta)$ by following the above procedure  
\begin{eqnarray}
{C^*_{Poly-2}}''(\eta)=-(\eta(1-\eta))^2=-(\eta^2+\eta^4-2\eta^3).
\end{eqnarray} 
From this we have
\begin{eqnarray}
{C_{Poly-2}^*}'(\eta)=-(\frac{1}{3}\eta^3+\frac{1}{5}\eta^5-\frac{2}{4}\eta^4) + K_1.
\end{eqnarray}
Satisfying ${C_{Poly-2}^*}'(\frac{1}{2})=0$ we find $K_1=\frac{1}{60}$. Therefore,
\begin{eqnarray}
C_{Poly-2}^*(\eta)=K_2(-\frac{1}{12}\eta^4 -\frac{1}{30}\eta^6 +\frac{1}{10}\eta^5 +\frac{1}{60}\eta).
\end{eqnarray}
Satisfying $C_{Poly-2}^*(\frac{1}{2})=\frac{1}{2}$ we find $K_2=\frac{960}{11}$.

\begin{figure*}[t]
  \centering
     \includegraphics[width=4in]{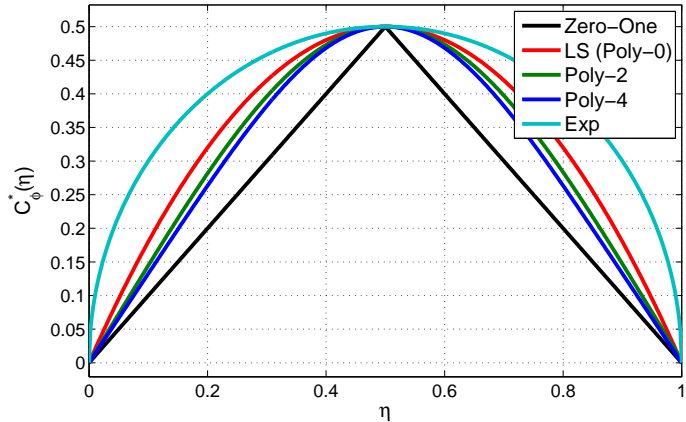} 
 \caption{\protect{\footnotesize{Plot of $C_{Poly-n}^*(\eta)$.}}}
  \label{fig:PlotPolinomialJ}
\end{figure*}

Figure-\ref{fig:PlotPolinomialJ} plots $C_{Poly-2}^*(\eta)$ which shows that, as expected, it is a closer approximation to $C_{0/1}^*(\eta)$ when compared to            $C_{LS}^*(\eta)$. Following the same steps, it is readily shown  that $C_{LS}^*(\eta)=C_{Poly-0}^*(\eta)$, meaning that $C_{LS}^*(\eta)$ is derived from the special case of $n=0$. 

As we increase $n$, we increase the order of the resulting polynomial which provides a tighter fit to $C_{0/1}^*(\eta)$. Figure-\ref{fig:PlotPolinomialJ} also plots $C_{Poly-4}^*(\eta)$ 
\begin{eqnarray}
\label{eq:CPoly4}
&&C_{Poly-4}^*(\eta)= \\
&&1671.3(-\frac{1}{90}\eta^{10} +\frac{1}{18}\eta^9 -\frac{3}{28}\eta^8 +\frac{2}{21}\eta^7 -\frac{1}{30}\eta^6 +\frac{1}{1260}\eta)  \nonumber
\end{eqnarray}
which is an even closer approximation to $C_{0/1}^*(\eta)$. 
Table-\ref{tab:RefFormulasDiffJPolyn} shows the corresponding minimum-risk $R_{C_{Poly-n}^*}(p^*)$ for different $C_{Poly-n}^*(\eta)$ functions, with 
$R_{C_{Poly-4}^*}(p^*)$ providing
the tightest bound on the Bayes error.  Arbitrarily tighter bounds are possible by simply using $C_{Poly-n}^*(\eta)$ with larger $n$.

\begin{table}[tbp]
\centering
\caption{\protect\footnotesize{Minimum-risk  for different $C_{Poly-n}^*(\eta)$ }}
\resizebox{\textwidth}{!}{ 
\begin{tabular}{|c|c|}
    \hline
    $C^*_{\phi}(\eta)$ & $R_{C_{\phi}^*}$ \\
		\hline
    \hline
    Zero-One & Bayes Error \\
    &  \\
    \hline
    Poly-0 (LS) & $\int \frac{P({\bf x}|1)P({\bf x}|-1)}{P({\bf x}|1)+P({\bf x}|-1)} d{\bf x}$ \\
    &  \\
    \hline
    Poly-2 & $\frac{K_2}{2} \int -\frac{P({\bf x}|1)^4}{12(2P({\bf x}))^3} - \frac{P({\bf x}|1)^6}{30(2P({\bf x}))^5} +\frac{P({\bf x}|1)^5}{10(2P({\bf x}))^4} +K_1P({\bf x}|1) d{\bf x}   $ \\
    & $K_1=0.0167,K_2=87.0196,P({\bf x})=\frac{P({\bf x}|1)+P({\bf x}|-1)}{2}$ \\
    & \\
    \hline
    Poly-4 & $\frac{K_2}{2} \int -\frac{P({\bf x}|1)^{10}}{90(2P({\bf x}))^9} +\frac{P({\bf x}|1)^9}{18(2P({\bf x}))^8} -\frac{3P({\bf x}|1)^8}{28(2P({\bf x}))^7}     +\frac{2P({\bf x}|1)^7}{21(2P({\bf x}))^6} -\frac{P({\bf x}|1)^6}{30(2P({\bf x}))^5}+K_1P({\bf x}|1)  d{\bf x}$ \\
    & $K_1=7.9365\times10^{-4},K_2=1671.3,P({\bf x})=\frac{P({\bf x}|1)+P({\bf x}|-1)}{2}$ \\
    & \\
    \hline
\end{tabular}
}
\label{tab:RefFormulasDiffJPolyn}
\end{table}

Such arbitrarily tight bounds on the Bayes error  are important in a number of applications such as  in feature selection and extraction \cite{NUNOMaxDiversityNIPS, NunoNaturalFeatures,NewFeatPers,FeatMutual}, information theory \cite{Fano, fdivrisk, book:InfoTheory, minimaxrisk}, classification and regression \cite{KLboost,ProjectClass,FriedmanPersuit}, etc. 
In the experiments section we specifically show how using $C^*_{\phi}$ with tighter bounds on the Bayes error results in better performance on a feature selection and classification problem. We then consider the effect of using projection vectors $\Phi({\bf w})$ that are more discriminative, such as the normalized kernel Fisher discriminant projection vector or normalized kernel SVM projection vector described in Lemma {\ref{thm:validchoices}}, rather than the canonical projection vector 
of $\Phi({\bf w})=\frac{({\bm \mu}_P-{\bm \mu}_Q)}{||({\bm \mu}_P-{\bm \mu}_Q)||_{\cal H}}$. We show that these more discriminative projection vectors $\Phi({\bf w})$  result in significantly improved performance on a set of kernel hypothesis testing experiments.

\section{Computing The Kernel Score and Kernel Divergence in Practice }
In most applications the distributions of $P$ and $Q$ are not directly known and are solely  represented through a set of sample points. We assume that the data points 
$\{{\bf x}_1, ..., {\bf x}_{n_1}\}$ are  sampled from  $P$ and the data points $\{{\bf x}_1, ..., {\bf x}_{n_2}\}$ are sampled from $Q$.
Note that the Kernel Score can be  written as
\begin{eqnarray}
S_{C^*_{\phi},k,{\cal F},{\Phi({\bf w})}}(P,Q)={\mathbf E}_{Z}\left[ C\left( \frac{P^p(x^p)}{P^p(x^p) + Q^p(x^p)}  \right)  \right],
\end{eqnarray}
where the expectation is over the distribution defined by $P_Z(z)=\frac{P^p(x^p) + Q^p(x^p)}{2}$. The empirical Kernel Score and empirical Kernel Divergence can now be written as
\begin{eqnarray}
\label{KScoreEMP}
{\hat S}_{C^*_{\phi},k,{\cal F},{\Phi({\bf w})}}(P,Q)&=&\frac{1}{n}\sum_{i=1}^{n}  C\left( \frac{P^p(x^p_i)}{P^p(x^p_i) + Q^p(x^p_i)}  \right)  \\
\label{KDEMP}
\widehat{KD}_{C,k,{\cal F},{\Phi({\bf w})}}(P,Q) &=& \frac{1}{2} -  {\hat S}_{C,k,{\cal F},{\Phi({\bf w})}}(P,Q),
\end{eqnarray}
where $n=n_1+n_2$ and $x^p_i$ is the projection of $\Phi({\bf x}_i)$ onto $\Phi({\bf w})$.

Calculating $x^p_i$ in the above formulation using equation (\ref{eq:FindxpwK}) is still not possible because we generally don't know  $\Phi({\bf w})$ and ${\bf w}$.
A similar problem exists for the MMD. Nevertheless the MMD \cite{TwoMMD} is estimated in practice as 
\begin{eqnarray}
\label{eq:MMDPrac}
&&\widehat{MMD}_{\cal F}(P,Q) = ||\hat{\bm \mu}_P - \hat{\bm \mu}_Q||^2_{\cal H} \\
&=&\frac{1}{n_1n_1}\sum_{i=1}^{n_1}\sum_{j=1}^{n_1}K({\bf x}_i{\bf x}_j) -\frac{2}{n_1n_2}\sum_{i=1}^{n_1}\sum_{j=1}^{n_2}K({\bf x}_i{\bf x}_j) + \frac{1}{n_2n_2}\sum_{i=1}^{n_2}\sum_{j=1}^{n_2}K({\bf x}_i{\bf x}_j).
\end{eqnarray}
In view of Lemma \ref{lemma:AlternateMMD} the MMD can be equivalently estimated  as
\begin{eqnarray}
\label{eq:MMDaltPrac1}
\widehat{MMD}_{\cal F}(P,Q)=(\hat \mu^p_P - \hat \mu^p_Q)^2,
\end{eqnarray}
where 
\begin{eqnarray}
\label{eq:MMDaltPrac2}
\hat \mu^p_P=\frac{\frac{1}{n_1n_1}\sum_{i=1}^{n_1}\sum_{j=1}^{n_1}K({\bf x}_i{\bf x}_j) -\frac{1}{n_1n_2}\sum_{i=1}^{n_1}\sum_{j=1}^{n_2}K({\bf x}_i{\bf x}_j)}{T},
\end{eqnarray}
\begin{eqnarray}
\label{eq:MMDaltPrac3}
\hat \mu^p_Q=\frac{\frac{1}{n_1n_2}\sum_{i=1}^{n_1}\sum_{j=1}^{n_2}K({\bf x}_i{\bf x}_j) - \frac{1}{n_2n_2}\sum_{i=1}^{n_2}\sum_{j=1}^{n_2}K({\bf x}_i{\bf x}_j)}{T}
\end{eqnarray}
and
\begin{eqnarray}
\label{eq:MMDaltPrac4}
T=\sqrt{\frac{1}{n_1n_1}\sum_{i=1}^{n_1}\sum_{j=1}^{n_1}K({\bf x}_i{\bf x}_j) -\frac{2}{n_1n_2}\sum_{i=1}^{n_1}\sum_{j=1}^{n_2}K({\bf x}_i{\bf x}_j) + \frac{1}{n_2n_2}\sum_{i=1}^{n_2}\sum_{j=1}^{n_2}K({\bf x}_i{\bf x}_j)}.
\end{eqnarray}
It can easily be verified that equations (\ref{eq:MMDaltPrac1})-(\ref{eq:MMDaltPrac4}) and equation (\ref{eq:MMDPrac}) are equivalent. 
This equivalent method for calculating the MMD can be elaborated as projecting the two embedded sample sets onto  $\Phi({\bf w})=\frac{(\hat{\bm \mu}_P-\hat{\bm \mu}_Q)}{||(\hat{\bm \mu}_P-\hat{\bm \mu}_Q)||}$, estimating   the  means $\hat\mu^p_P$ and $\hat\mu^p_Q$ of the  projected embedded sample sets
and then finding the distance between these estimated  means. This  might seem like over
complicating the original  procedure. Yet, it serves to show that the  MMD is solely measuring the distance between the means while  disregarding all the other information available regarding the projected embedded distributions. Similarly, assuming that 
$\Phi({\bf w})=\frac{(\hat{\bm \mu}_P-\hat{\bm \mu}_Q)}{||(\hat{\bm \mu}_P-\hat{\bm \mu}_Q)||}$, $x^p_i$ can now be estimated as
\begin{eqnarray}
\label{eq:projctedxpiPrac}
x^p_i&&\!\!\!\!\!\!=<\Phi({\bf x}_i),{\bf w}>=<\Phi({\bf x}_i),\frac{({\hat{\bm \mu}_P}-{\hat{\bm \mu}_Q})}{||({\hat{\bm \mu}_P}-{\hat{\bm \mu}_Q})||}> \\
&&=\frac{<\Phi({\bf x}_i),{\hat{\bm \mu}_P}>-<\Phi({\bf x}_i),{\hat{\bm \mu}_Q}>}{||({\hat{\bm \mu}_P}-{\hat{\bm \mu}_Q})||} \\
&&=\frac{\frac{1}{n_1}\sum_{j=1}^{n_1}<\Phi({\bf x}_i),\Phi({\bf x}_j)>-\frac{1}{n_2}\sum_{j=1}^{n_2}<\Phi({\bf x}_i),\Phi({\bf x}_j)>}{T} \\
\label{eq:projctedxpiPracLAST}
&&=\frac{\frac{1}{n_1}\sum_{j=1}^{n_1}K({\bf x}_i,{\bf x}_j)-\frac{1}{n_2}\sum_{j=1}^{n_2}K({\bf x}_i,{\bf x}_j)}{T}.
\end{eqnarray}
Once the $x^p_i$ are found for all $i$ using equation (\ref{eq:projctedxpiPracLAST}), estimating other statistics  such as the variance is trivial. For example, the   variances of the projected embedded distributions can now be estimated as
\begin{eqnarray}
\label{eq:ProjectedVariance11}
(\hat \sigma^p_P)^2=\frac{1}{n_1}\sum_{i=1}^{n_1}(x^p_i-\hat \mu^p_P)^2 \\
\label{eq:ProjectedVariance22}
(\hat \sigma^p_Q)^2=\frac{1}{n_2}\sum_{i=1}^{n_2}(x^p_i-\hat \mu^p_Q)^2.
\end{eqnarray}
In light of this, the empirical Bhattacharyya Kernel Score and empirical Bhattacharyya Kernel Divergence can now be readily calculated in practice as
\begin{eqnarray}
\label{eq:battarExp2EMP}
&&\hat S_{C,k,{\cal F},{\Phi({\bf w})}}(P,Q)= \frac{1}{2}\cdot e^{(B)}, \\
&&\widehat{KD}_{C_{Exp},k,{\cal F},{\Phi({\bf w})}}(P,Q)=\frac{1}{2} -\frac{1}{2}\cdot e^{(B)}, \\
&& \hat B=\frac{1}{4}\log \left ( \frac{1}{4} \left (\frac{( \hat \sigma^p_P)^2}{( \hat \sigma^p_Q)^2}+\frac{( \hat \sigma^p_Q)^2}{( \hat \sigma^p_P)^2}+2 \right)\right )+\frac{1}{4} \left ( \frac{( \hat \mu^p_P- \hat \mu^p_Q)^2}{( \hat \sigma^p_P)^2 + ( \hat \sigma^p_Q)^2} \right).
\end{eqnarray} 
Finally, the empirical Kernel Score of equation (\ref{KScoreEMP})  and the empirical Kernel Divergence of equation (\ref{KDEMP}) can be calculated in practice after finding $P^p(x^p_i)$ and $Q^p(x^p_i)$ using  any one dimensional probability model.

Note that in the above formulations we  used the canonical $\Phi({\bf w})=\frac{(\hat{\bm \mu}_P-\hat{\bm \mu}_Q)}{||(\hat{\bm \mu}_P-\hat{\bm \mu}_Q)||}$.
A similar approach is possible  for other valid choices of $\Phi({\bf w})$. 
Namely, the projection vector $\Phi({\bf w})$ associated with the  kernel Fisher discriminant 
can be found in the form of
\begin{eqnarray}
\Phi({\bf w})=\sum_{j=1}^{n} \alpha_j \Phi({\bf x}_j)
\end{eqnarray}
using Algorithm 5.16 in \cite{Book:KernelMethodsforPatternAnalysis}. In this case $x^p_i$  can be found as
\begin{eqnarray}
\label{eq:xpiLDAEMP}
x^p_i=\frac{<\Phi({\bf w}), \Phi({\bf x}_i)>}{||\Phi({\bf w})||} = \frac{\sum_{j=1}^{n} \alpha_j K({\bf x}_i,{\bf x}_j)}{||\Phi({\bf w})||},
\end{eqnarray}
where
\begin{eqnarray}
||\Phi({\bf w})|| = \sqrt{\sum_{i=1}^{n}\sum_{j=1}^{n} \alpha_i \alpha_j K({\bf x}_i,{\bf x}_j)}.
\end{eqnarray}

The projection vector $\Phi({\bf w})$ associated  with the kernel SVM   can  also be found in the form of
\begin{eqnarray}
\Phi({\bf w})=\sum_{j \in SV} \alpha_j \phi({\bf x_j})
\end{eqnarray}
using standard algorithms \cite{SVMSolver, SVMGuide}, where $SV$ is the set of support vectors. In this case the  $x^p_i$ can be found using 
equation (\ref{eq:xpiLDAEMP}) calculated over the support vectors.

\section{One-Class Classifier for Kernel Hypothesis Testing}	
\label{sec:OneClass}
From Theorem \ref{thm:StrictDivProperty} we conclude that the Kernel Divergence is   injective similar to the MMD. This means that the Kernel Divergence can 
be directly thresholded and used in hypothesis testing. We showed that while the MMD simply measures the distance between the means of the projected embedded distributions, the Bhattacharyya Kernel Divergence (BKD) incorporates information about both the means and  variances of the two projected embedded distributions. 
We also showed that in general the  Kernel Divergence (KD) provides a measure related to the minimum risk of the two projected embedded distributions. 
Each one of these measures takes into account a different aspect of the two projected embedded distributions in relation to each other. We can integrate all of these measures into our hypothesis test by constructing a vector where each element is a different measure and learn a one-class classifier for this vector.  In the hypothesis testing experiments of Section \ref{sec:HypoTests}, we constructed the vectors [MMD, KD] and [MMD, BKD]  
and implemented a simple one-class nearest neighbor classifier with infinity norm ~\citep{OneClassPhd} as depicted in Figure \ref{fig:RejectRegion}.

\begin{figure*}[t]
  \centering
		\includegraphics[width=3in]{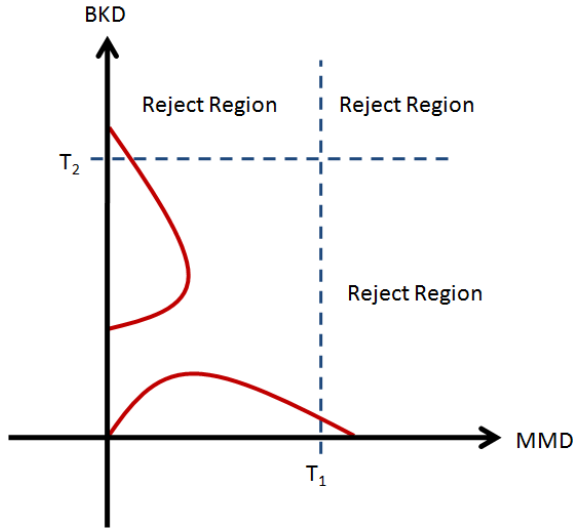}
 \caption{\protect{\footnotesize{Hypothesis testing using a one-class nearest neighbor classifier with infinity norm. The thresholds $T_1$ and $T_2$,  on the MMD and BKD axis, define the region where the null hypothesis is rejected. }}}
  \label{fig:RejectRegion}
\end{figure*}

\section{Experiments}
In this section we include various experiments that confirm different theoretical aspects of the Kernel Score and Kernel Divergence.

\subsection{Feature selection experiments}
Different bounds on the Bayes error are used in feature selection and ranking algorithms
 \cite{NUNOMaxDiversityNIPS, NunoNaturalFeatures,NewFeatPers,FeatMutual, FeatRankEntropy}. In this section we show that the tighter bounds we have derived, namely $C^*_{Poly-2}$ and $C^*_{Poly-4}$, allow for improved feature selection and ranking. The experiments used ten binary UCI data sets of relatively small size: (\#1)  Habermanӳ survival,(\#2) original Wisconsin breast cancer , (\#3) tic-tac-toe , (\#4) sonar,
(\#5) Pima-diabetes , (\#6) liver disorder , (\#7) Cleveland heart disease , (\#8) echo-cardiogram , (\#9) breast cancer prognostic, and (\#10) breast cancer diagnostic.

Each data set was split into five folds, four of which were used for training and one for testing. This created five
train-test pairs per data set, over which the results were averaged. The original data was augmented with noisy features. This was done by taking each feature and adding random scaled noise to a certain percentage of the data points.  The scale parameters were $\{ 0.1, 0.3\}$ and the percentage of data points that were randomly affected was  $\{ 0.25, 0.50, 0.75 \}$. Specifically, for each feature, a percentage of the data points had scaled zero mean Gaussian noise added to that feature in the form of
\begin{eqnarray}
{\bf x}_i={\bf x}_i+{\bf x}_i \cdot {y} \cdot s,
\end{eqnarray} 
where ${\bf x}_i$ is the $i$-th  feature of the original data vector, ${y} \in N(0,1)$ is the Gaussian noise  and $s$ is the scale parameter. 
The empirical minimum risk  was then computed for each feature where $P_{{\bf X}|Y}({
\bf x}|y)$ was modeled as a $10$ bin histogram.

A greedy feature selection algorithm was implemented in which the features were ranked according to their empirical minimum risk and the highest ranked $5\%$ and $10\%$ of the features were selected. The selected features were then used to train and test a linear SVM classifier.
If a certain minimum risk  $C^*_{\phi}$ is a better bound on the Bayes error, we would expect it to choose better features and these better features should translate into a better SVM classifier with smaller error rate on the test data. Five different $C^*_{\phi}$ were considered namely $C^*_{Poly-4}$, $C^*_{Poly-2}$, $C^*_{LS}$, $C^*_{Log}$ and $C^*_{Exp}$ and the error rate corresponding to  each $C^*_{\phi}$ was computed and averaged over the five folds. The average error rates were then ranked such that a rank of $1$ was assigned to the $C^*_{\phi}$ with smallest error and a rank of $5$ assigned to the $C^*_{\phi}$ with largest error.

The  \emph{rank over selected features}  was computed by averaging the ranks found by using both $5\%$ and $10\%$ of the highest ranked  features. This process was repeated a total of   $25$ times for each UCI data set and the \emph{over all average rank} was found by averaging 
over the $25$ experiment runs. 
The  \emph{over all average rank} found for each UCI data set and each $C^*_{\phi}$ is reported in Table-\ref{tab:OverAllFeatRanks}. The last two columns of this table are the total number of times each $C^*_{\phi}$ has had the best rank over the ten different data sets (\#W) and a ranking of the \emph{over all average rank} computed for each data set and then averaged across all data sets (Rank). It can be seen that  $C^*_{Poly-4}$ which was designed to have the tightest bound on the Bayes error has the most number of wins of $4$ and smallest Rank of $2.4$ while $C^*_{Exp}$ which has the loosest bound on the Bayes error has the least number of wins of $0$ and worst Rank of  $3.75$. As expected, the Rank  for each $C^*_{\phi}$ is   in order of how tightly they approximate the Bayes error with in order $C^*_{Poly-4}$, $C^*_{Poly-2}$ and $C^*_{LS}$ at the top and $C^*_{Log}$ and $C^*_{Exp}$ at the bottom. This is in accordance with the discussion of Section~\ref{sec:TighterBounds22}.

\begin{table}[tbp]
\centering
\caption{\protect\footnotesize{The  \emph{over all average rank}  for each UCI data set and each $C^*_{\phi}$. }}
\resizebox{\textwidth}{!}{ 
\begin{tabular}{|c|c|c|c|c|c|c|c|c|c|c|c|c|  }
    \hline
    $C^*_{\phi}$ & \#1&\#2&\#3&\#4&\#5&\#6&\#7&\#8&\#9&\#10 & \#W & Rank \\
    \hline
    $C^*_{Poly-4}$ & 2.9 & 2.75 & 3.62 & 2.77 & {\bf 2.86} & {\bf 2.39} & 3.55 & {\bf 2.87} & 3.25 & {\bf 2.86} & 4 & 2.4 \\
    \hline
    $C^*_{Poly-2}$ & {\bf 2.82} & 2.88 & 3.37 & {\bf 2.62} & 2.87 & 2.74 & 3.27 & 2.9 & 3.46 & 2.98 & 2 & 2.7 \\
    \hline
    $C^*_{LS}$  & 3.02 & {\bf 2.73} & 2.92 & 3.03 & 2.9 & 3.17 & {\bf 2.65} & 2.96 & 2.97 & 3.16 & 2 & 2.8 \\
    \hline
    $C^*_{Log}$ & 3.16   & 3.32 & {\bf 2.5}  & 3.44  & 3.0 & 3.2 & 2.78 & 3.08 & {\bf 2.62} & 2.88 & 2 & 3.35 \\
    \hline
    $C^*_{Exp}$ & 3.1  & 3.32  & 2.59 & 3.14  & 3.37 & 3.5 & 2.75 & 3.19 & 2.7 & 3.12 & 0 & 3.75 \\
    \hline
\end{tabular}
}
\label{tab:OverAllFeatRanks}
\end{table}

\subsection{ Kernel Hypothesis Testing Experiments }
\label{sec:HypoTests}
The first set of experiments comprised of  hypothesis tests on Gaussian samples. Specifically,  two hypothesis tests were considered. In the first test, we used $250$ 
samples for each $25$ dimensional Gaussian distribution ${\mathcal N}(0,1.5I)$ and ${\mathcal N}(0,1.7I)$. Note that the means are equal and the variances are slightly different.  In the second test, we used $250$ samples for each $25$ dimensional Gaussian distribution ${\mathcal N}(0,1.5I)$ and ${\mathcal N}(0.1,1.5I)$. Note that the variances are equal and the means are slightly different.  In both cases the reject thresholds were found from $100$ bootstrap iterations for a fixed type-I error 
of $\alpha=0.05$. We used the Gaussian kernel embedding for all experiments and the kernel parameter was found using the median heuristic of \cite{TwoMMD}. Also,
the Kernel Divergence (KD) used  $C^*_{Poly-4}$ of equation (\ref{eq:CPoly4}) and  one dimensional Gaussian distribution models.
Unlike the classification problem described in the previous section, having a tight  estimate of the Bayes error is not important for hypothesis testing experiments and so the actual concave $C$ function used is not crucial. 
The type-II error test results for $100$ repetitions are reported in Table \ref{tab:GaussMMDHypoTest} for the MMD, BKD, KD methods, where ${\Phi(\bf w})=\frac{({{\bm \mu}_P}-{{\bm \mu}_Q})}{||({{\bm \mu}_P}-{{\bm \mu}_Q})||}$, along with the  combined method described in Section \ref{sec:OneClass} where  a one-class nearest neighbor classifier with infinity norm is learned for [MMD, KD] and [MMD, BKD].
These  results are  typical  and in   general  (a) the KD and BKD methods do better than the MMD when the means are equal and the variances are different, (b) the MMD does better than the KD and BKD when the variances are equal and the means are different and (c) The combined methods of  
[MMD, KD] and [MMD, BKD] do well for both cases. We usually don't know which case we are dealing with in practice and so the combined methods
of [MMD, KD] and [MMD, BKD] are preferred. 

\begin{table}[tbp]
\centering
\caption{\protect\footnotesize{Percentage of type-II error for the hypothesis tests on two types of Gaussian samples given $\alpha=0.05$.   }}
\begin{tabular}{|c|c||c|  }
    \hline
    Method & $\sigma_1=1.5$, $\sigma_2=1.7$   & $\mu_1=0$, $\mu_2=0.1$    \\
		       &  $\mu_1=\mu_2=0$ & $\sigma_1=\sigma_2=1.5$ \\
    \hline
		MMD & 46 & 25 \\
    \hline
		KD & 13 & 42 \\
    \hline
		BKD & {\bf 11} & 40 \\
		\hline
		[MMD, KD] & 13 & 25 \\
		\hline
		[MMD, BKD] & 12 & {\bf 24} \\
		\hline
\end{tabular}
\label{tab:GaussMMDHypoTest}
\end{table}

\subsubsection{Bench-Mark Gene Data Sets}
Next we evaluated the proposed methods on a group of high dimensional bench-mark gene data sets. 
The data sets are detailed in Table \ref{tab:GeneDataNUMS} and are challenging given their small sample size and high dimensionality.  
The hypothesis testing involved splitting the positive samples in two and using the first half to learn the reject thresholds from $1000$ bootstrap iterations for a fixed type-I error of $\alpha=0.05$. We  used the Gaussian kernel embedding for all experiments and the kernel parameter was found using the median heuristic of \cite{TwoMMD}. The Kernel Divergence (KD) used  $C^*_{Poly-4}$ of equation (\ref{eq:CPoly4}) and  one dimensional Gaussian distribution models.
The type-II error test results for $1000$ repetitions are reported in Table \ref{tab:GeneDataResultsHypo} for the MMD, BKD, KD, [MMD, KD] and [MMD, BKD] methods. Also, three projection directions are considered namely,  MEANS where $\Phi({\bf w})=\frac{({{\bm \mu}_P}-{{\bm \mu}_Q})}{||({{\bm \mu}_P}-{{\bm \mu}_Q})||}$,
 FISHER where the $\Phi({\bf w})$ associated with the kernel Fisher linear discriminant is used, and SVM where the $\Phi({\bf w})$  associated with the kernel SVM is used. 

We have  reported the rank of each method among the five methods with the same projection direction under RANK1  and the overall rank among all fifteen methods under RANK2 in the last column. Note that the first row of Table \ref{tab:GeneDataResultsHypo} with MMD distance measure and MEANS projection direction is the only  method previously proposed in the literature \cite{MMD}. We should also note that the KD with FISHER projection direction encountered numerical problems in the form of very small variance estimates, which resulted in  poor performance. 
Nevertheless, we can see that in general the KD and BKD methods  which incorporate more information regarding the projected distributions, outperform the MMD. 
Second, using more discriminant projection directions like the FISHER or SVM outperform simply projecting onto MEANS.
Finally, the [MMD, KD] and [MMD, BKD] methods that combine the information
provided by both the MMD and the KD or BKD have the  lowest ranks. Specifically, the  [MMD, KD] with SVM projection direction has the overall lowest rank among all fifteen methods.


\begin{table}[tbp]
\centering
\caption{\protect\footnotesize{Gene data set details.  }}
\resizebox{\textwidth}{!}{ 
\begin{tabular}{|c|c|c|c|c|  }
    \hline
    Number & Data Set & \#Positives & \#Negatives & \#Dimensions  \\
    \hline
		\#1& Lung Cancer Womenӳ Hospital & 31  & 150 & 12533  \\
    \hline
		\#2& Lukemia & 25 & 47 & 7129  \\
    \hline
		\#3& Lymphoma Harvard Outcome & 26 & 32 & 7129 \\
    \hline
		\#4& Lymphoma Harvard & 19 & 58 & 7129 \\
    \hline
		\#5& Central Nervous System Tumor & 21 & 39 & 7129  \\
    \hline
		\#6& Colon Tumor & 22 & 40 & 2000  \\
    \hline
		\#7& Breast Cancer ER & 25 & 24 & 7129  \\
    \hline
\end{tabular}
}
\label{tab:GeneDataNUMS}
\end{table}

%

\begin{table}[tbp]
\centering
\caption{\protect\footnotesize{  Percentage of type-II error for the gene data sets  given $\alpha=0.05$. RANK1 is the  rank of each method among the five methods with the same projection direction  and RANK2 is the overall rank among all fifteen methods.}}
\resizebox{\textwidth}{!}{ 
\begin{tabular}{|c|c||c|c|c|c|c|c|c||c|c|  }
    \hline
    Projection & Measure & \#7 & \#6 & \#5 & \#4 & \#3 & \#2 & \#1 & Rank1  & Rank2 \\
    \hline
		MEANS &MMD & 24.3 & 27.4 & 95 & 31.2 & 90.8 & 11.7 & 6.3 &  3.42 & 9.14\\
    \hline
		MEANS &KD & 9.8 & 58.5 & 83.8 & 53.1 & 79.2 & 64.7 & 7.7 & 3.71  & 10.14\\
    \hline
		MEANS &BKD & 12 & 56.9 & 83.4 & 52.5 & 79.3 & 58.0 & 3.7 & 3.14  & 9.14\\
    \hline
		MEANS &[MMD, KD] & 12.2 & 48 & 82.9 & 25.2 & 84.1 & 14.7 & 3.0 & 2.57  & 7.42\\
    \hline
		MEANS & [MMD, BKD] & 13.2 & 47.3 & 81.9 & 24.3 & 83.6 & 14.0 & 3.2 & {\bf 2.14}  & 6.42\\
		
    \hline
		\hline
		
		FISHER& MMD & 5.8 & 26.5 & 90.2 & 24.8 & 83.1 & 14.1 & 4.2 & {\bf 1.78}  & 6.07\\
    \hline
		FISHER&KD & 100 & 100 & 100 & 100 & 100 & 100 & 100 & 5 & 15\\
    \hline
		FISHER&BKD & 30.6 & 52.4 & 82.4 & 66.0 & {\bf 64.0} & 73.3 & 22.6 & 3.14  & 10.28\\
    \hline
		FISHER&[MMD, KD] & 9.6 & 26.5 & 95.3 & 31.9 & 93.6 & 18.7 & 5.4 & 3.14  & 9.42\\
    \hline
		FISHER&[MMD, BKD] & 6.2 & {\bf 26.4} & 82.8 & 31.0 & 74.9 & 18.6 & 5.4 & 1.92  & 5.21\\
		
		\hline
		\hline
		
		SVM& MMD & 22.8 & 29.9 & 95.1 & 26.2 & 89.2 & {\bf 10.0} & 2.1 & 3.85  & 8.00\\
    \hline
		SVM&KD & {\bf 4.0} & 48.2 & {\bf 81.3} & 33.4 & 82.4 & 41.1 & 1.0 & 3.28  & 6.42\\
    \hline
		SVM&BKD & 4.3 & 44.2 & 86.3 & 32.0 & 79.4 & 34.4 & 0.5 & 2.85  & 6.57\\
    \hline
		SVM&[MMD, KD] & 6.3 & 28.1 & 88.4 & {\bf 20.5} & 86.4 & 13.7 & {\bf 0.4} & {\bf 2.28}  & {\bf 5.14}\\
    \hline
		SVM&[MMD, BKD] & 6.6 & 28.2 & 89.0 & {\bf 20.5} & 84.9 & 13.8 & {\bf 0.4} & 2.71  & 5.57\\
		\hline
\end{tabular}
}
\label{tab:GeneDataResultsHypo}
\end{table}

\section{Conclusion}
While we have concentrated on the hypothesis testing problem in the experiments section, we envision many different applications  for the Kernel Score and Kernel Divergence.  We showed that the MMD is a special case of the Kernel Score and so the Kernel Score can now be used in all other applications based on the MMD, such as integrating biological data, imitation learning, etc.   We also showed that the Kernel Score is related to the minimum risk of the projected embedded distributions and we showed how to derive tighter bounds on the Bayes error. Many applications that are based on risk minimization, bounds on the Bayes error or divergence measures such as classification, regression, feature selection, estimation, information theory etc, can now use the Kernel Score and Kernel Divergence to their benefit. We presented the Kernel Score as a general formulation for a score function in the Reproducing Kernel Hilbert Space and considered when it has the important property of being  strictly proper. The Kernel Score is thus also directly applicable to probability elicitation, forecasting, finance and meteorology which rely on strictly proper scoring rules.

\bigskip
\begin{center}
{\large\bf SUPPLEMENTARY MATERIAL}
\end{center}

\section{Proof of Theorem \ref{thm:StrictDivProperty}}
\label{app:KscoreStrict}
If $P=Q$ then 
\begin{eqnarray}
KD_{C,k,{\cal F},{\Phi({\bf w})}}(P,Q)=-\int P^p(z)C(\frac{1}{2})dz + \frac{1}{2} = -\frac{1}{2}\int P^p(z)dz + \frac{1}{2}= 0.
\end{eqnarray}
Next, we prove the converse. The proof is identical to Theorem 5 of \cite{MMD} up to the point where we must prove that if $KD_{C,k,{\cal F},{\Phi({\bf w})}}(P,Q)=0$ then
${\bm \mu}_P={\bm \mu}_Q$. To show this we write
\begin{eqnarray}
KD_{C,k,{\cal F},{\Phi({\bf w})}}(P,Q)=-\int  \left( \frac{P^p(z) + Q^p(z)}{2} \right) C\left( \frac{P^p(z)}{P^p(z) + Q^p(z)}  \right)dz +\frac{1}{2} =0
\end{eqnarray}
or
\begin{eqnarray}
\int  \left( \frac{P^p(z) + Q^p(z)}{2} \right) C\left( \frac{P^p(z)}{P^p(z) + Q^p(z)}  \right)dz =\frac{1}{2}.
\end{eqnarray}
Since $C(\eta)$ is concave and has a maximum value of $\frac{1}{2}$ at $\eta=\frac{1}{2}$ then the above equation can only hold if  

\begin{eqnarray}
C\left(\frac{P^p(z)}{P^p(z)+Q^p(z)}\right)=\frac{1}{2},
\end{eqnarray}
which means that 
\begin{eqnarray}
\frac{P^p(z)}{P^p(z)+Q^p(z)}=\frac{1}{2},
\end{eqnarray}
and so
\begin{eqnarray}
P^p(z)=Q^p(z).
\end{eqnarray}
From this we conclude that their associated means must be equal, namely
\begin{eqnarray}
\mu^p_P=\mu^p_Q.
\end{eqnarray} 
The above equation can be written as
\begin{eqnarray}
<{\bm \mu}_P, \Phi({\bf w})>_{\cal H}=<{\bm \mu}_Q, \Phi({\bf w})>_{\cal H}
\end{eqnarray} 
or equivalently as
\begin{eqnarray}
<{\bm \mu}_P-{\bm \mu}_P, \Phi({\bf w})>_{\cal H}=0.
\end{eqnarray} 
Since $\Phi({\bf w})$ is not in the orthogonal compliment of ${\bm \mu}_P-{\bm \mu}_P$ then it must be that
\begin{eqnarray}
{\bm \mu}_P={\bm \mu}_Q.
\end{eqnarray}
The rest of the proof is again identical to Theorem 5 of \cite{MMD} and the theorem is similarly proven.

To prove that the Kernel Score is strictly proper we note that if $P=Q$ then $KD_{C,k,{\cal F},{\Phi({\bf w})}}(Q,Q)=0$ and 
so $S_{C,k,{\cal F},{\Phi({\bf w})}}(Q,Q)=\frac{1}{2}$. This means  that 
we need to show that $S_{C,k,{\cal F},{\Phi({\bf w})}}(Q,Q)=\frac{1}{2} \ge S_{C,k,{\cal F},{\Phi({\bf w})}}(P,Q)$. This readily follows since $C(\eta)$ is strictly concave with maximum at  $C(\frac{1}{2})=\frac{1}{2}$.

\section{Proof of Lemma \ref{thm:validchoices}}
\label{app:ValidVect}
$\Phi({\bf w})=\frac{({\bm \mu}_P-{\bm \mu}_Q)}{||({\bm \mu}_P-{\bm \mu}_Q)||_{\cal H}}$ is not in the orthogonal compliment of $M=\{{\bm \mu}_P - {\bm \mu}_Q\}$
since 
\begin{eqnarray}
<\frac{({\bm \mu}_P-{\bm \mu}_Q)}{||({\bm \mu}_P-{\bm \mu}_Q)||_{\cal H}}, {\bm \mu}_P - {\bm \mu}_Q>_{\cal H} \ne 0.
\end{eqnarray}

The $\Phi({\bf w})$ equal to the kernel Fisher discriminant projection vector is not in the orthogonal compliment of $M=\{{\bm \mu}_P - {\bm \mu}_Q\}$
because if it were then the kernel Fisher discriminant objective, which can be written as $\frac{\mu^p_P - \mu^p_Q}{(\sigma^p_P)^2 + (\sigma^p_Q)^2}$, 
would not be maximized and would instead be equal to zero.

The $\Phi({\bf w})$ equal to the kernel SVM projection vector is not in the orthogonal compliment of $M=\{{\bm \mu}_P - {\bm \mu}_Q\}$ since the kernel SVM 
 is  equivalent to the kernel Fisher discriminant computed on the set of support vectors ~\cite{SVDequalLDA}.

\section{Proof of Lemma \ref{lemma:AlternateMMD}}
\label{app:MMDNewLook}
We know that $\mu^p_P$ is the projection of ${\bm \mu}_P$ onto $\Phi({\bf w})$ so we can write
\begin{eqnarray}
\mu^p_P=<{\bm \mu}_P, \Phi({\bf w})>_{\cal H}=<{\bm \mu}_P, \frac{({\bm \mu}_P-{\bm \mu}_Q)}{||({\bm \mu}_P-{\bm \mu}_Q)||_{\cal H}}>_{\cal H}
=\frac{<{\bm \mu}_P,{\bm \mu}_P>_{\cal H}-<{\bm \mu}_P,{\bm \mu}_Q>_{\cal H}}{||({\bm \mu}_P-{\bm \mu}_Q)||_{\cal H}}
\end{eqnarray}
Similarly,
\begin{eqnarray}
\mu^p_Q=<{\bm \mu}_Q, \Phi({\bf w})>_{\cal H}=<{\bm \mu}_Q, \frac{({\bm \mu}_P-{\bm \mu}_Q)}{||({\bm \mu}_P-{\bm \mu}_Q)||_{\cal H}}>_{\cal H}
=\frac{<{\bm \mu}_Q,{\bm \mu}_P>_{\cal H}-<{\bm \mu}_Q,{\bm \mu}_Q>_{\cal H}}{||({\bm \mu}_P-{\bm \mu}_Q)||_{\cal H}}.
\end{eqnarray}
Hence,
\begin{eqnarray}
&&(\mu^p_P - \mu^p_Q)^2=\left(\frac{<{\bm \mu}_P,{\bm \mu}_P>_{\cal H}-2<{\bm \mu}_P,{\bm \mu}_Q>_{\cal H}+<{\bm \mu}_Q,{\bm \mu}_Q>_{\cal H}}{||({\bm \mu}_P-{\bm \mu}_Q)||_{\cal H}} \right)^2 \\
&&= \left(\frac{<({\bm \mu}_P-{\bm \mu}_Q),({\bm \mu}_P-{\bm \mu}_Q)>_{\cal H}}{||({\bm \mu}_P-{\bm \mu}_Q)||_{\cal H}}\right)^2 
= \left(\frac{||({\bm \mu}_P-{\bm \mu}_Q)||_{\cal H}^2}{||({\bm \mu}_P-{\bm \mu}_Q)||_{\cal H}}\right)^2 = ||({\bm \mu}_P-{\bm \mu}_Q)||_{\cal H}^2.
\end{eqnarray}

\section{Proof of Theorem \ref{Thm:KScoreMMDConnection}}
\label{app:MMDKSCoreConnection}
The result readily follows by setting $MMD_{\cal F}(P,Q)=(\mu^p_P - \mu^p_Q)^2$ and $\sigma^p_P = \sigma^p_Q$ into equation (\ref{eq:battarExp2}).

\section{Proof of Lemma \ref{Thm:Brier}}
\label{app:Brier}
By  adding and subtracting $\int_{\bf x} P_{{\bf X}}({\bf x}) \left[ P_{{\bf Y}|{\bf X}}(1|{\bf x})\phi({ p^*}({\bf x})) + P_{{\bf Y}|{\bf X}}(-1|{\bf x})\phi(-{ p^*}({\bf x}))  \right] d{\bf x}$ and considering equation (\ref{equ:MinRiskInit}),
the  risk $R({\hat p^*})$ can be written as
\begin{eqnarray}
R({\hat p^*}) &=& E_{{\bf X},Y}[\phi(y{\hat p^*}({\bf x}))]=\int_{\bf x} P_{{\bf X}}({\bf x}) \sum_y P_{{\bf Y}|{\bf X}}(y|{\bf x})\phi(y{\hat p^*}({\bf x})) d{\bf x} \\ 
  &=& \int_{\bf x} P_{{\bf X}}({\bf x}) \left[ P_{{\bf Y}|{\bf X}}(1|{\bf x})\phi({\hat p^*}({\bf x})) + P_{{\bf Y}|{\bf X}}(-1|{\bf x})\phi(-{\hat p^*}({\bf x}))  \right] d{\bf x} \nonumber \\
	&=& \int_{\bf x} P_{{\bf X}}({\bf x}) \left[ P_{{\bf Y}|{\bf X}}(1|{\bf x})\left(\phi({\hat p^*}({\bf x}))-\phi({ p^*}({\bf x}))\right) + P_{{\bf Y}|{\bf X}}(-1|{\bf x})\left(\phi(-{\hat p^*}({\bf x}))-\phi(-{ p^*}({\bf x})) \right) \right] d{\bf x} \nonumber \\
	&+& \int_{\bf x} P_{{\bf X}}({\bf x}) \left[ P_{{\bf Y}|{\bf X}}(1|{\bf x})\phi({ p^*}({\bf x})) + P_{{\bf Y}|{\bf X}}(-1|{\bf x})\phi(-{ p^*}({\bf x}))  \right] d{\bf x} \nonumber \\
&=&	R_{Calibration} + R({p^*}).
\end{eqnarray}
 

The first term denoted  $R_{Calibration}$  is obviously zero if we have a perfectly calibrated predictor such that ${\hat p^*}({\bf x}) ={p^*}({\bf x})$
for all ${\bf x}$ and is thus a measure of calibration.
Finally, using  equation $\eta({\bf x})=P_{Y|{\bf X}}(1|{\bf x})={[f_\phi^*]}^{-1}(p^*({\bf x})) $ and Theorem \ref{Thm:HamedNuno}, the minimum risk term $R({p^*})$ can be written as
\begin{eqnarray}
&&R({p^*}) = \int_{\bf x} P_{{\bf X}}({\bf x}) \left[ P_{{\bf Y}|{\bf X}}(1|{\bf x})\phi({ p^*}({\bf x})) + P_{{\bf Y}|{\bf X}}(-1|{\bf x})\phi(-{ p^*}({\bf x}))  \right] d{\bf x} \\
&=& \int_{\bf x} P_{{\bf X}}({\bf x}) [ \eta({\bf x}) C_\phi^*(\eta({\bf x}) ) + \eta({\bf x})(1-\eta({\bf x})) [C_\phi^*]^\prime(\eta({\bf x})) \\
&+& (1-\eta({\bf x})) C_\phi^*((1-\eta({\bf x})) ) + (1-\eta({\bf x}))(\eta({\bf x})) [C_\phi^*]^\prime((1-\eta({\bf x}))) ] d{\bf x} \\
&=& \int_{\bf x} P_{{\bf X}}({\bf x}) [ \eta({\bf x}) C_\phi^*(\eta({\bf x}) ) + \eta({\bf x})(1-\eta({\bf x})) [C_\phi^*]^\prime\left(\eta({\bf x})\right) \\    
&+&  C_\phi^*(\eta({\bf x}) ) - \eta({\bf x}) C_\phi^*(\eta({\bf x}) ) - \eta({\bf x})(1-\eta({\bf x})) [C_\phi^*]^\prime\left(\eta({\bf x})\right)  ] d{\bf x} \\
&=& \int_{\bf x} P_{{\bf X}}({\bf x})  C_\phi^*(\eta({\bf x}) ) d{\bf x} \\
&=& \int_{\bf x} P_{{\bf X}}({\bf x})  C_\phi^*({[f_\phi^*]}^{-1}(p^*({\bf x})) ) d{\bf x} \\
&=& \int_{\bf x} P_{{\bf X}}({\bf x})  C_\phi^*(P_{Y|{\bf X}}(1|{\bf x}) ) d{\bf x} \\
\end{eqnarray}

\section{Proof of Theorem \ref{Thm:KDMinRiskConect}}
\label{app:KscoreMinRisk}
Assuming equal priors $P_Y(1)=P_Y(-1)=\frac{1}{2}$,
\begin{eqnarray}
\label{eq:EqualPriorPX}
P_{{\bf X}}({\bf x}) =\frac{P_{{\bf X}|Y}({\bf x}|1) + P_{{\bf X}|Y}({\bf x}|-1)}{2}
\end{eqnarray}
and 
\begin{eqnarray}
\label{eq:P1givenx}
P_{Y|{\bf X}}(1|{\bf x})=\frac{P_{{\bf X}|Y}({\bf x}|1)}{P_{{\bf X}|Y}({\bf x}|1)+P_{{\bf X}|Y}({\bf x}|-1)}.
\end{eqnarray}
We can now write the minimum risk  as
\begin{eqnarray}
\label{eq:SrefUnder}
R({p^*}) =  
	\int_{X} \left(\frac{P_{{\bf X}|Y}({\bf x}|1)+P_{{\bf X}|Y}({\bf x}|-1)}{2}\right) C_\phi^*\left(\frac{P_{{\bf X}|Y}({\bf x}|1)}{P_{{\bf X}|Y}({\bf x}|1)+P_{{\bf X}|Y}({\bf x}|-1)}\right) d{\bf x}
\end{eqnarray}
Equation (\ref{eq:KDMinRiskequality}) readily follows by setting $P_{{\bf X}|Y}({\bf x}|1)=P^p$ and $P_{{\bf X}|Y}({\bf x}|-1)=Q^p$, in which case
$R({p^*})$ is $R^p({p^*})$.

\section{Proof of Theorem \ref{Thm:TightestBoundCR}}
\label{app:Polyn}
The symmetry requirement of $C_{\phi}^*(\eta)=C_{\phi}^*(1-\eta)$ results in a similar requirement on the second derivative ${C_{\phi}^*}''(\eta)={C_{\phi}^*}''(1-\eta)$ and concavity requires that the second derivative satisfy ${C_{\phi}^*}''(\eta)<0$. 
The symmetry  and concavity constraints can both be satisfied by considering
\begin{eqnarray}
{C_{Poly-n}^*}''(\eta) \propto -(\eta(1-\eta))^n.
\end{eqnarray}
From this we  write
\begin{eqnarray}
{C_{Poly-n}^*}'(\eta) \propto \int-(\eta(1-\eta))^n d(\eta) + K_1 = Q(\eta)+K_1.
\end{eqnarray}
Satisfying the constraint that ${C_{Poly-n}^*}'(\frac{1}{2})=0$, we  find $K_1$ as
\begin{eqnarray}
K_1 = -Q(\frac{1}{2}).
\end{eqnarray}
Finally, $C_{Poly-n}^*(\eta)$ is
\begin{eqnarray}
C_{Poly-n}^*(\eta)=K_2(\int  Q(\eta) d(\eta) +K_1\eta),
\end{eqnarray}
where 
\begin{eqnarray}
K_2=\frac{\frac{1}{2}}{\left. (\int  Q(\eta) d(\eta) +K_1\eta)\right|_{\eta=\frac{1}{2}}}
\end{eqnarray}
is a scaling factor such that $C_{Poly-n}^*(\frac{1}{2})=\frac{1}{2}$.
$C_{Poly-n}^*(\eta)$ meets all the requirements of Theorem \ref{thm:tightestC}  so   
$C_{Poly-n}^*(\eta) \ge  C_{0/1}^*(\eta)$ for all $\eta \in [0~1]$ and 
$R_{C_{Poly-n}^*}  \ge R^*$. 

Next, we need to prove that if we follow the above procedure for $n+1$ and find $C_{Poly-(n+1)}^*(\eta)$ then $R_{C_{Poly-n}^*} \ge R_{C_{Poly-(n+1)}^*}$. 
We accomplish this   by showing that $C_{Poly-n}^*(\eta) \ge C_{Poly-(n+1)}^*(\eta)$. 
Without loss of generality, let 
\begin{eqnarray}
{C_{Poly-n}^*}''(\eta) = -(\eta(1-\eta))^n
\end{eqnarray}
 and 
\begin{eqnarray}
{C_{Poly-(n+1)}^*}''(\eta) = -(\eta(1-\eta))^{n+1}
\end{eqnarray}
then ${C_{Poly-n}^*}''(\eta) \le {C_{Poly-(n+1)}^*}''(\eta)$
since  $\eta \in [0~1]$. Also, since ${C_{Poly-n}^*}''(\eta)<0$ and ${C_{Poly-n}^*}''(\eta)={C_{Poly-n}^*}''(1-\eta)$ then 
${C_{Poly-n}^*}'(\eta)$ is a monotonically decreasing function and ${C_{Poly-n}^*}'(\eta)=-{C_{Poly-n}^*}'(1-\eta)$ and so ${C_{Poly-n}^*}'(\frac{1}{2})=0$.
From the mean value theorem 
\begin{eqnarray}
{C_{Poly-n}^*}''(c_1)={C_{Poly-n}^*}'(\frac{1}{2})-{C_{Poly-n}^*}'(\eta)=-{C_{Poly-n}^*}'(\eta)
\end{eqnarray}
and
\begin{eqnarray}
{C_{Poly-(n+1)}^*}''(c_2)={C_{Poly-(n+1)}^*}'(\frac{1}{2})-{C_{Poly-(n+1)}^*}'(\eta)=-{C_{Poly-(n+1)}^*}'(\eta)
\end{eqnarray}
for any $0 \le \eta \le \frac{1}{2}$ and some $0 \le c_1 \le \frac{1}{2}$ and $0 \le c_2 \le \frac{1}{2}$.
Since ${C_{Poly-n}^*}''(\eta) \le {C_{Poly-(n+1)}^*}''(\eta)$ for all $\eta \in [0~1]$ then ${C_{Poly-n}^*}''(c_1) \le {C_{Poly-(n+1)}^*}''(c_2)$
and so 
\begin{eqnarray}
{C_{Poly-n}^*}'(\eta) \ge {C_{Poly-(n+1)}^*}'(\eta)
\end{eqnarray}
for all $0 \le \eta \le \frac{1}{2}$. 
A similar argument leads to
\begin{eqnarray}
{C_{Poly-n}^*}'(\eta) \le {C_{Poly-(n+1)}^*}'(\eta)
\end{eqnarray}
for all $\frac{1}{2} \le \eta \le 1$.

Since ${C_{Poly-n}^*}'(\frac{1}{2})=0$ and ${C_{Poly-n}^*}''(\eta) \le 0$ then $C_{Poly-n}^*(\eta)$ has a maximum at $\eta=\frac{1}{2}$.
Also, since $C_{Poly-n}^*(\eta)$ is a polynomial of $\eta$ with no constant term, then $C_{Poly-n}^*(0)=0$ and because of symmetry $C_{Poly-n}^*(1)=0$. 
From the mean value theorem 
\begin{eqnarray}
{C_{Poly-n}^*}'(c_1)=C_{Poly-n}^*(\eta)-C_{Poly-n}^*(0)=C_{Poly-n}^*(\eta)
\end{eqnarray}
and
\begin{eqnarray}
{C_{Poly-(n+1)}^*}'(c_2)=C_{Poly-(n+1)}^*(\eta)-C_{Poly-(n+1)}^*(0)=C_{Poly-(n+1)}^*(\eta)
\end{eqnarray}
for any $0 \le \eta \le \frac{1}{2}$ and some $0 \le c_1 \le \frac{1}{2}$ and $0 \le c_2 \le \frac{1}{2}$.
Since ${C_{Poly-n}^*}'(\eta) \ge {C_{Poly-(n+1)}^*}'(\eta)$ for all $0 \le \eta \le \frac{1}{2}$ then ${C_{Poly-n}^*}'(c_1) \ge {C_{Poly-(n+1)}^*}'(c_2)$
and so 
\begin{eqnarray}
C_{Poly-n}^*(\eta) \ge C_{Poly-(n+1)}^*(\eta)
\end{eqnarray}
for all $0 \le \eta \le \frac{1}{2}$. 
A similar argument leads to
\begin{eqnarray}
C_{Poly-n}^*(\eta) \ge C_{Poly-(n+1)}^*(\eta)
\end{eqnarray}
for all $\frac{1}{2} \le \eta \le 1$.
Finally, since $C_{Poly-n}^*(\eta)$ and $C_{Poly-(n+1)}^*(\eta)$ are concave functions with  maximum at $\eta=\frac{1}{2}$, scaling these functions by $K_2$ and $K_2'$ respectively, so that their maximum is equal to $\frac{1}{2}$ will not change the final result of
\begin{eqnarray}
C_{Poly-n}^*(\eta) \ge C_{Poly-(n+1)}^*(\eta)
\end{eqnarray}
for all  $0 \le \eta \le 1$.

Finally, to show that $R_{C_{Poly-n}^*}$ converges to $R^*$  we need to show that $C_{Poly-n}^*(\eta)$ converges  
to $C_{0/1}^*(\eta)=\min\{\eta, 1-\eta\}$  as $n \rightarrow \infty$. 
We can expand $\int Q(\eta) d(\eta)$  
and write $C_{Poly-n}^*(\eta)$ as
\begin{eqnarray}
C_{Poly-n}^*(\eta)=K_2(a_1\eta^{(2n+2)} + a_2\eta^{(2n+2)-1} + a_3\eta^{(2n+2)-2} ... + a_{n+1}\eta^{(2n+2)-n} + K_1\eta).
\end{eqnarray}
Assuming that $0 \le \eta \le \frac{1}{2}$ then 
\begin{eqnarray}
\lim_{n\rightarrow \infty} C_{Poly-n}^*(\eta) = K^{\top}_2( 0 + K_1^{\top}\eta) =  K^{\top}_2 K_1^{\top} \eta,
\end{eqnarray}
where $K_1^{\top} =\lim_{n\rightarrow \infty} K_1$ and $K_2^{\top}=\lim_{n\rightarrow \infty} K_2$. 
Since 
\begin{eqnarray}
K_1K_2=\frac{-\frac{1}{2}Q(\frac{1}{2})}{(\int Q(\eta)d(\eta) - Q(\frac{1}{2})\eta)|_{\eta=\frac{1}{2}}}
\end{eqnarray}
then 
\begin{eqnarray}
K_1^{\top}K_2^{\top}= \frac{-\frac{1}{2}Q(\frac{1}{2})}{( 0 - \frac{1}{2}Q(\frac{1}{2}) )} = 1.
\end{eqnarray}
So, we can write 
\begin{eqnarray}
\lim_{n\rightarrow \infty} C_{Poly-n}^*(\eta) =  K^{\top}_2 K_1^{\top} \eta =  \eta . 
\end{eqnarray}
A similar argument for $\frac{1}{2} \le \eta \le 1$ completes the convergence proof.

\bibliographystyle{Chicago}
\bibliography{IEEEexample}
\end{document}